 \documentclass[final, 12pt]{colt2020-arxiv} 

\title[Active Local Learning]{Active Local Learning}
\usepackage{times}

\coltauthor{%
 \Name{Arturs Backurs} \Email{backurs@ttic.edu}\\
 \addr Toyota Technological Institute at Chicago
 \AND
 \Name{Avrim Blum} \Email{avrim@ttic.edu}\\
 \addr Toyota Technological Institute at Chicago%
 \AND
 \Name{Neha Gupta} \Email{nehagupta@cs.stanford.edu}\\
 \addr Stanford University%
}
\usepackage{thm-restate}
\usepackage{amssymb}
\usepackage{comment}
\usepackage{hyperref}
\usepackage{cleveref}

\newcommand{\R}{\mathbb{R}}
\newcommand*\centermathcell[1]{\omit\hfil$\displaystyle#1$\hfil\ignorespaces}

\DeclareMathOperator*{\E}{\mathbb{E}}
\usepackage{algorithm}
\usepackage{algorithmic}
\usepackage{caption}
\DeclareMathOperator*{\argmin}{argmin}
\usepackage{float}
\usepackage{graphicx, nicefrac}
\usepackage{thmtools, thm-restate}
\usepackage{color-edits}

\newcommand{\p}{\mathcal{P}}
\newcommand{\pshort}{\mathcal{P}_{\sf sh}}
\newcommand{\plong}{\mathcal{P}_{\sf lg}}
\newcommand{\psh}{{p}_{\sf sh}}
\newcommand{\plg}{{p}_{\sf lg}}

\newcommand{\plongone}{\mathcal{P}_{\sf{lg}, 1}}
\newcommand{\plongtwo}{{\mathcal{P}}_{\sf{lg}, 2}}
\newcommand{\plongthree}{{\mathcal{P}}_{\sf{lg}, 3}}
\newcommand{\D}{\mathcal{D}}
\newcommand{\Dsh}{\mathcal{D}_{\sf{sh}}}
\newcommand{\Dlg}{\mathcal{D}_{\sf{lg}}}
\newcommand{\Fl}{\mathcal{F}_L}
\newcommand{\ft}{\tilde{f}}
\newcommand{\alg}{\text{Alg}}
\newcommand{\diffd}{\Delta_{\D}}
\newcommand{\errd}{\text{err}_\D}
\newcommand{\err}{\text{err}}
\newcommand{\diff}{\Delta}
\newcommand{\defn}{:\,=}

\newcommand{\1}{\ensuremath{{\sf (i)}}}
\newcommand{\2}{\ensuremath{{\sf (ii)}}}
\newcommand{\3}{\ensuremath{{\sf (iii)}}}

\newcommand{\zt}{\tilde{z}}

\begin{document}

\maketitle

\begin{abstract}%
In this work we consider active {\em local learning}: given a query point $x$, and active access to an unlabeled training set $S$, output the prediction $h(x)$ of a near-optimal $h \in H$ using significantly fewer labels than would be needed to actually learn $h$ fully. In particular, the number of label queries should be independent of the complexity of $H$, and the function $h$ should be well-defined, independent of $x$.  This immediately also implies an algorithm for {\em distance estimation}: estimating the value $opt(H)$ from many fewer labels than needed to actually learn a near-optimal $h \in H$, by running local learning on a few random query points and computing the average error.

For the hypothesis class consisting of functions supported on the interval $[0,1]$ with Lipschitz constant bounded by $L$, we present an algorithm that makes $O(({1 / \epsilon^6}) \log(1/\epsilon))$ label queries from an unlabeled pool of $O(({L / \epsilon^4})\log(1/\epsilon))$ samples. It estimates the distance to the best hypothesis in the class to an additive error of $\epsilon$ for an arbitrary underlying distribution. We further generalize our algorithm to more than one dimensions.
We emphasize that the number of labels used is independent of the complexity of the hypothesis class which is linear in $L$ in the one-dimensional case. Furthermore, we give an algorithm to locally estimate the values of a near-optimal function at a few query points of interest with number of labels independent of $L$.

We also consider the related problem of approximating the minimum error that can be achieved by the Nadaraya-Watson estimator under a linear diagonal transformation with eigenvalues coming from a small range. For a $d$-dimensional pointset of size $N$, our algorithm achieves an additive approximation of $\epsilon$, makes $\tilde{O}({d}/{\epsilon^2})$ queries and runs in $\tilde{O}({d^2}/{\epsilon^{d+4}}+{dN}/{\epsilon^2})$ time.
\end{abstract}

\section{Introduction}
%
%
Consider a setting where we have a large amount of unlabeled data but the corresponding labels are expensive to obtain. Our aim is to understand what information can we reliably obtain about the predictor that we would have learned had we been given unlimited labeled data by actively querying only a few labels (``few'' $=$ independent of the complexity of the hypothesis class).

In particular, we look at the following two related questions. First, given an unlabeled set of training data sampled from a distribution, is it possible to estimate how well the best prediction function in a hypothesis class would do when the number of labels that we can actually obtain is insufficient to learn the best prediction function accurately? Second, suppose we have a few queries of interest of which we are interested in the labels of, is it possible to output the predictions for those queries corresponding to a nearly optimal function in the class without running the full training algorithm. Note that a nearly optimal function corresponds to a function which has low total error on data with respect to the underlying distribution. There are many natural scenarios in which this could be useful. For example, consider a setting where we are interested in predicting the outcome of a treatment for a particular patient with a certain disease. Medical records of patients who had similar treatments in the past are available in a hospital database but since the treatment outcome is sensitive and private, we want to minimize the number of patients for whom we acquire the information. Note that in this case, we cannot directly acquire the label of the query patient because the treatment procedure is an intervention which cannot be undone. Formally, we are interested in computing the labels of a few particular queries ``locally" and ``parallelly" with few label queries (independent of the complexity measure of the hypothesis class). We would like to mention that both the aforementioned problems are related because if we have an algorithm for predicting the labels of a few test points with few label queries, we can sample a few data points and use their labels to estimate the error of the best function in the class. While the question of error estimation has been previously studied in certain settings \citep{kong2018estimating, blum2018active}, there has been no prior work for local and parallel predictions for a function class with global constraints in learning based settings to the best of our knowledge.

In this work, we answer these questions in affirmative. For the class of Lipschitz functions with Lipschitz constant at most $L$, we show that it is possible to estimate the error of the best function in the class with respect to an underlying distribution with independent of $L$ label queries. We also show that it is possible to locally estimate the values of a nearly optimal function at a few test points of interest with independent of $L$ label queries. A key point to notice here is that a function can predict any arbitrary values for the fixed constant number of queries and still be approximately optimal in terms of the total error since the queries  have a zero measure with respect to the underlying distribution. Therefore, the additional guarantee that we have is after a common preprocessing step, the combined function, obtained when the algorithm is run in parallel for all possible query points independently, is $L$-Lipschitz and approximately optimal in terms of total error with respect to the underlying distribution. 

We also give an algorithm for estimating the minimum error for Nadaraya-Watson prediction algorithm amongst a set of linear transformations which is efficient in terms of both sample and time complexity. Note that computing the prediction for even a single query point requires computing a weighted sum of the labels of the training points and hence requires $Nd$ time and $N$ labels where $N$ is the number of points in the training set and $d$ is the dimension of the data points. Moreover, naively computing the error for a set of linear transformations with size exponential in $d$ would require exponential in $d$ labels and time which has a multiplicative factor of $N$ and ${1}/{\epsilon^d}$ where $\epsilon$ is the accuracy parameter. However, we obtain sample complexity which only depends polynomially on $d$ and logarithmically on $N$ and time complexity without the multiplicative dependence of $N$ and $1/\epsilon^d$.

We would like to clarify how the setting considered in this paper differs from the classical notion of local learning algorithms. The notion of local learning algorithms \citep{bottou1992local} is used to refer to learning schemes where the prediction at a point depends locally on the training points around it. The key distinction is that rather than proposing a local learning strategy and seeing how well it performs, we are looking at the question of whether we can find local algorithms for simulating empirical risk minimization for a hypothesis class with global constraints.

The remainder of the paper is organized as follows. We begin in Section~\ref{section:results} with a high level overview of our results and discuss the related work in Section~\ref{section:related-work}. Section~\ref{section:lipschitz} is devoted to our main results for the class of one dimensional Lipschitz functions. In Section~\ref{section:kde}, we present our results for error estimation for Nadaraya-Watson estimator. Finally, we end with a conclusion and some possible future directions in Section~\ref{section:conclusion}. 
\section{Our Results}
\label{section:results}
Consider a class of one dimensional Lipschitz functions supported on the interval $[0,1]$ with Lipschitz constant at most $L$ and data drawn from an arbitrary unknown distribution $\mathcal{D}$ with labels in the range $[0,1]$. For this setting, we show (Theorem~\ref{thm:lipschitzlocalquery}) that there exists a function $\tilde{f}$ in the class that is  optimal up to an additive error of $\epsilon$ such that for any query $x$, $\tilde{f}(x)$ can be computed with $O(({1}/{\epsilon^4})\log(1/\epsilon))$ label queries to a pool of $O(({L}/{\epsilon^4})\log(1/\epsilon))$ unlabeled samples drawn from the distribution $\mathcal{D}$. Also, the function values $\tilde{f}(x)$ at these query points $x$ can be computed in parallel once the unlabeled random samples have been drawn and fixed beforehand. Note that standard empirical risk minimization approaches would require a sample complexity of $O({L}/{\epsilon^3})$ to output the value of an approximately optimal function even at a single query point.\includecomment{Using standard generalization arguments, one can learn a function f with additive epsilon error approximation to the best function with O(frac{L}{epsilon^3}) labeled samples but in this case, even if we know that we are interested in computing a function value at a fixed given query point x, this would still require solving a dynamic program or a linear program with all the samples}But, the number of samples required by our approach for constant number of queries is independent of $L$ (which determines the complexity of the function class required for learning). In this setting, we think of $L$ as large compared to $\epsilon$ which is the error parameter. At a high level, we show that it is possible to effectively reduce the hypothesis class of bounded Lipschitz functions to a strictly smaller class of piece-wise independent Lipschitz functions where the function value can be computed locally by not losing too much in terms of the total accuracy.

We also show (Theorem~\ref{thm:lipschitzerror}) that for the class of $L$-Lipschitz functions considered above, it is possible to estimate the error of the optimal function in the class up to an additive error $\epsilon$ using $O(({1}/{\epsilon^6})\log(1/\epsilon))$ active label queries over an unlabeled pool of size $O(({L}/{\epsilon^4})\log(1/\epsilon))$. The idea is to compute the empirical error of the local function $\tilde{f}(x)$ constructed above by using $O({1}/{\epsilon^2})$ random samples from distribution $D$. Since $\tilde{f}(x)$ can be computed locally for a given query $x$, the total number of labels needed is independent of $L$. Using standard concentration results and the fact that $\tilde{f}$ is $\epsilon$-optimal, we get the desired estimate. 

We also extend the results to the case of more than one dimensions where the dimension is constant with respect to the Lipschitz constant $L$. The results are mentioned in Theorems~\ref{thm:lipschitzlocalquery-highd} and~\ref{thm:lipschitzerror-highd}. 

For the related setting of Nadaraya-Watson estimator, we show (Theorem~\ref{thm:kde-min-error}) that it is possible to estimate the minimum error that can be achieved under a linear diagonal transformation with eigenvalues bounded in a small range with additive error at most $\epsilon$ by making $\tilde{O}({d}/{\epsilon^2})$ label queries over a $d$-dimension unlabeled training set with size $N$ in running time $\tilde{O}({d^2}/{\epsilon^{d+4}}+{dN}/{\epsilon^2})$. Note that exactly computing the prediction for even a single data point requires going over the entire dataset thus needing $N$ label queries and $Nd$ time. Moreover, naively computing the error for each of the linear diagonal transformation with bounded eigenvalues would require number of labels depending on ${1}/{\epsilon^d}$ and a running time depending multiplicatively on $N$ and ${1}/{\epsilon^d}$. In comparison, we achieve a labeled query complexity independent of $N$ and polynomial dependence on the dimension $d$. Moreover, we separate the multiplicative dependence of $N$ and ${1}/{\epsilon^d}$ in the running time to an additive dependence.
We will further elaborate on our algorithm and the comparison with the standard algorithm in Section~\ref{section:kde}.
%
\section{Related Work}
\label{section:related-work}
\textbf{Local Computation Algorithms. } Our work on locally learning Lipschitz functions closely resembles the concept of local computation algorithms introduced in the pioneering work of \cite{rubinfeld2011fast}. They were interested in the question of whether it is possible to compute specific parts of the output in time sublinear in the input size. As mentioned by the authors, that work was a formalization of different forms of this concept already existing in literature in the form of  local distributed computation, local algorithms,
locally decodable codes and local reconstruction models. The reader can refer to \cite{rubinfeld2011fast} for the detailed discussion of work in each of these subfields. In the paper, the authors looked at several graph problems in the local computation model like maximal independent set, $k$-SAT and hypergraph coloring. Since then, there has been a lot of further work on local computation algorithms for various graph problems including maximum matching, load balancing, set cover and many others (\cite{mansour2013local, alon2012space, mansour2012converting, parnas2007approximating,parter2019local,levi2014local, grunau2020improved}). There also has been work on solving linear systems locally in sublinear time \citep{andoni2018solving} for the special case of sparse symmetric diagonally dominant matrices with small condition number. However, computing the best Lipschitz function for a given set of data cannot be written as a linear system. Moreover, the primary focus of all of these works has been on sublinear computational complexity, whereas we focus primarily on sample complexity.

In another related work, \cite{mansour2014robust} used local computation algorithms in the context of robust inference to give polynomial time algorithms. They formulated their inference problem as an exponentially sized linear program and showed that the linear program (LP) has a special structure which allowed them to compute the values of certain variables in the optimal solution of the linear program in time sublinear in the total number of variables. They did this by sampling a polynomial number of constraints in the LP.  Note that for our setting for learning Lipschitz functions, given all the unlabeled samples, learning the value of the best Lipschitz function on a particular input query can be cast as a linear program. However, our LP does not belong to the special class of programs that they consider. Moreover, we have a continuous domain and the number of possible queries is infinite. We cannot hope to get a globally Lipschitz solution by locally solving a smaller LP with constraints sampled independently for each query. We have to carefully design the local intervals and use different learning strategies for different types of intervals to ensure that the learned function is globally Lipschitz and also has good error bounds. %

In another work, \cite{feige2015learning} considered the use of local computation algorithms for inference settings. They reduced their problem of inference for a particular query to the problem of computing minimum vertex cover in a bipartite graph. However, the focus was on time complexity rather than sample complexity. The problem of computing a part of output of the learned function in sublinear sample complexity as compared to the usual notion of complexity of the function class (VC-dimension, covering numbers) has not been previously looked at in the literature to the best of our knowledge. 

\textbf{Transductive Inference. }Another related line of work that has been done in the learning theory community is based on transductive inference \citep{vapnik1998statistical} which as opposed to inductive inference aims to estimate the values of the function on a few given data points of interest. The philosophy behind this line of work is to avoid solving a more general problem as an intermediate step to solve a problem. This idea is very similar to the idea considered here wherein we are interested in computing the prediction at specific query points of interest. However, our prediction algorithm still requires a guarantee on the total error over the complete domain with respect to the underlying distribution though it is never explicitly constructed for the entire domain unless required. The reader can refer to chapters 24 and 25 in \cite{chapelle2009semi} for additional discussion on the topic. 

\textbf{Local Learning Algorithms. } The term local learning algorithms has been used to refer to the class of learning schemes where the prediction at a test point depends only on the training points in the vicinity of the point such as k-nearest neighbor schemes \citep{bottou1992local,vapnik1992principles}. However, these works have primarily focused on proposing different local learning strategies and evaluating how well they perform. In contrast, we are interested in the question of whether local algorithms can be used for simulating empirical risk minimization for a hypothesis class with global constraints such as the class of Lipschitz functions. 

\textbf{Property Testing.} There also has been a large body of work in the theoretical computer science community on property testing where the goal is to determine whether a function belongs to a class of functions or is far away from it with respect to some notion of distance in sublinear sample complexity. The commonly studied testing problems include testing monotonicity of a function over a boolean hypercube \citep{goldreich1998testing, chakrabarty2016n, chen2014new, khot2018monotonicity}, testing linearity over finite fields \citep{bellare1996linearity, ben2003randomness}, testing for concise DNF formulas \citep{diakonikolas2007testing, parnas2002testing} and testing for small decision trees \citep{diakonikolas2007testing}. However, all the aforementioned algorithms work in a query model where a query can be made on any arbitrary domain point of choice. 

The setting which is closer to learning  where the labels can only be obtained from a fixed distribution was first studied by \cite{goldreich1998property, kearns2000testing}.\includecomment{ kearns and rons 200 gave better sample complexity bounds for testing various problems like decisions trees and restricted form of neural networks than would be required for learning the best function in that class.} This setting is also called as passive property testing. The notion of active property testing was first introduced by \cite{balcan2012active} where an algorithm can make active label queries on an unlabeled sample of points drawn from an unknown distribution. However, one limitation of these algorithms is that they do not give meaningful bounds to distinguish the function from being approximately close to the function class (rather than belonging to the class) vs. far away from it. 

\cite{parnas2006tolerant} first introduced the notion of tolerant testing where the aim is to detect whether the function is $\epsilon$ close to the class or $2\epsilon$ far from it in the query model with queries on arbitrary domains points of choice. This also relates to estimating the distance of the function from the class within an additive error of $\epsilon$. \cite{blum2018active} first studied the problem of active tolerant testing where they were interested in algorithms which are tolerant to the function not being exactly in the function class and also have active query access to the labels over the unlabeled samples from an unknown distribution. Specifically, \cite{blum2018active}
gave algorithms for estimating the distance of the function from the class of union of $d$ intervals with a labeled sample complexity of $\text{poly}({1}/{\epsilon})$. The key point to be noted is that the labeled sample complexity is independent of $d$ which is the VC dimension of that class and dictates the number of samples required for learning. We note that our algorithm for error estimation for the class of Lipschitz functions in labeled sample complexity independent of $L$ is another work along these lines. 

Property testing has also been studied for the specific class of Lipschitz functions \citep{jha2013testing, chakrabarty2013optimal, berman2014lp}. However, all of these results either query arbitrary domain points or are in the non-tolerant setting or work only for discrete domain. 
%

A closely related work to our method of error estimation where the predicted label depends on the labels of the training data weighted according to some appropriate kernel function was also considered in \cite{blum2018active}. In particular, they looked at the setting where the predicted label is based on the $k$-nearest neighbors and showed that the $\ell_{1}$ loss can be estimated up to an additive error of $\epsilon$ using $O({1}/{\epsilon^2})$ label queries on $N+O({1}/{\epsilon^2})$ unlabeled samples. Their results also extend to the case where the prediction is the weighted average of all the unlabeled points in the sample by sampling with a probability proportional to the weight of the point. However, this sampling when repeated for $O({1}/{\epsilon^d})$ different linear transformations would lead to a labeled sample complexity depending on ${1}/{\epsilon^d}$. Moreover, sampling a point with a probability proportional to the weight would require a running time of $N$ per query and thus, would give a total running time of $O({N}/{\epsilon^{d+2}})$. 


\section{Lipschitz Functions}
\label{section:lipschitz}
In this section, we describe the formal problem setup for both local learning and error estimation for the class of Lipschitz functions. We then proceed to describe our algorithms and derive label query complexity bounds for them. We restrict our attention to the one-dimensional problems here and defer the details of the more than one dimensional setup to Appendix~\ref{subsection:lipschitz-highd}.



\subsection{Problem Setup}
\label{subsection:lipschitz-problemsetup}
For any fixed $L>0$, let $\mathcal{F}_{L}$ be the class of $d$ dimensional functions supported on the domain $[0,1]^d$ with Lipschitz constant at most $L$, that is,
\begin{equation}
    \mathcal{F}_L = \{f:[0,1]^d\mapsto [0,1], \  f \text{ is } L\text{-Lipschitz }\}.
\end{equation}
Let $\D$ be any distribution over $[0,1]^d\times[0,1]$ and $\D_x$ be the corresponding marginal distribution over $[0,1]^d$. The prediction error of a function $f$ with respect to $\D$ and the optimal prediction error of the class $\Fl$ are
\begin{equation*}
    \errd(f) \defn E_{x,y \sim \mathcal{D}}|y-f(x)| \quad \text{and} \quad \errd(\Fl) \defn \min_{f \in \mathcal{F}_L}E_{x,y \sim \mathcal{D}}|y-f(x)|.
\end{equation*}
Also, let us denote the error of a function $f$ relative to another function $f'$ and function class $\Fl$ as
\begin{equation}
    \diffd(f, f')\defn \errd(f) -\errd(f') \quad \text{and} \quad \diffd(f, \Fl) \defn  \errd(f) - \errd(\Fl).
\end{equation}
We say that a function $f\in \Fl$ is $\epsilon$-optimal with respect to distribution $\D$ and the function class $\Fl$ if it satisfies $\diffd(f, \Fl) \leq \epsilon$. Let functions $f^*_{\D} \in \Fl$ and $\hat{f}_S \in \Fl$ be the minimizers of error with respect to the distribution and  empirical error on set $S$ defined as
\begin{align}
f^*_{\D} = \argmin_{f \in \Fl } \E_{(x, y) \sim \D}|y-f(x)| \quad \text{ and } \quad 
\hat{f}_S = \argmin_{f \in \Fl } \sum_{(x_j, y_j) \in S}|y_j-f(x_j)|.\label{lipschitz:minemperror}
\end{align}

\paragraph{Local Learning.} Given access to unlabeled samples from $\D_x$ and a test point $x^*$, the objective of \emph{Local Learning} is to output the prediction $\ft(x^*)$ using a small number of label queries such that $\ft\in \Fl$ is \mbox{$\epsilon$-optimal}. In addition, we would like such a predictor, $\alg$, to be able to answer multiple such queries while being consistent with the same function $\ft$, that is,
\begin{equation*}
    \alg(x^*) = \ft(x^*) \quad \text{for all } x^* \in [0,1]^d.
\end{equation*}



\paragraph{Error Estimation.} Given access to unlabeled samples from distribution $\D_x$, the goal of \emph{Error Estimation} is to output an estimate of the optimal prediction error $\errd( \Fl)$ up to an additive error of $\epsilon$ using few label queries. 


\subsection{Guarantees for Local Learning}
\label{subsection:lipschitz-1d-localearning}
We begin by describing our proposed algorithm for Local Learning and then provide a bound on its query complexity in Theorem~\ref{thm:lipschitzerror}.

Our algorithm for local predictions first involves a preprocessing step (Algorithm~\ref{alg1}) which takes as input the Lipschitz constant $L$, sampling access to distribution $\D_x$ and the error parameter $\epsilon$ and returns a partition $\p= \{I_1 = [b_0, b_1], I_2=[b_1, b_2], \ldots, \}$ and a set $S$ of unlabeled samples. The partition $\p$ consists of alternating intervals of length ${1}/{L}$ and ${1}/(L\epsilon)$ over the domain $[0,1]$. Let us divide these intervals\footnote{Note that long intervals at the boundary could be shorter, but those can be handled similarly. Moreover, if any long interval gets more than $\frac{1}{2\epsilon^4}\log(\frac{1}{\epsilon})$ samples, we discard the future samples which fall into that interval.}  further into the two sets 
\begin{small}
\begin{equation*}
    \begin{gathered}
    \plong\defn\{[b_0,b_1],[b_2,b_3],\ldots,\}\  \text{(long intervals)}\ \text{ and } \ 
    \pshort\defn\{[b_1,b_2],[b_3,b_4],\ldots,\}\  \text{(short intervals)}.\\
    \end{gathered}
\end{equation*}
\end{small}
%
 \begin{algorithm}[t]
   \begin{algorithmic}[1]
\caption{Preprocess($L, \D_x, \epsilon$)\label{alg1}}
   \STATE Sample a uniformly random offset $b_1$ from $\{1,2,\cdots,\frac{1}{\epsilon}\}\frac{1}{L}$. 
   \STATE Divide the $[0,1]$  interval into alternating intervals of length $\frac{1}{L\epsilon}$ and $\frac{1}{L}$ with boundary at $b_1$ and let $\p$ be the resulting partition, that is, $\p = \{[b_0=0, b_1], [b_1, b_2], \ldots, \}$ where
   $b_2=b_1+\frac{1}{L}, b_3=b_2+\frac{1}{L\epsilon}, \ldots$. 
   \STATE Sample a set $S = \{x_i\}_{i=1}^M$ of $M = O(\frac{L}{\epsilon^4}\log(\frac{1}{\epsilon}))$ unlabeled examples from distribution $\D_x$.\\
   \STATE \textbf{Output} $S, \p$. 
   \end{algorithmic}
  \end{algorithm}
The Query algorithm (Algorithm~\ref{alg2}) for test point $x^*$ takes as input the set $S$ of unlabeled samples and the partition $\p$ returned by the Preprocess algorithm. Note that all subsequent queries use the same partition $\p$ and the same set of unlabeled examples $S$. The algorithm uses different learning strategies depending on whether $x^*$ belongs to one of the long intervals in $\plong$ or short intervals in $\pshort$. For the long intervals, it outputs the prediction corresponding to the empirical risk minimizer (ERM) function restricted to that interval. Whereas for the short interval, the prediction is made by linearly interpolating the function values at the boundaries with the neighbouring long intervals. This linear interpolation ensures that the overall function is Lipschitz. 
We bound the expected prediction error of this scheme with respect to class $\Fl$ by separately bounding this error for long and short intervals. For the long intervals, we prove that the ERM has low error by ensuring that each interval contains enough unlabeled samples. On the other hand, we show that the short intervals do not contribute much to the error because of their low probability under the distribution $\D$.
 
%
%

 \begin{algorithm}[H]
   \begin{algorithmic}[1]
\caption{Query($x, S, \p=\{[b_0,b_1],[b_1,b_2], [b_2,b_3], \ldots\}$)\label{alg2}}
    
   \IF{query $x \in I_i = [b_{i-1}, b_i]\text{ where } I_i \in \plong$}
   \STATE Query labels\footnotemark for $x \in S \cap I_i.$
   \STATE \textbf{Output} $\hat{f}_{S \cap I_i}(x)$.\\
    \ELSIF{query $x \in I_i = [b_{i-1}, b_i] \text{ where } I_i \in \pshort$}
      \STATE Query labels for $x \in S \cap (I_{i-1} \cup I_{i+1})$.
    \STATE \textbf{if} $b_{i-1} > 0$ \textbf{then } $v^l_i = \hat{f}_{S \cap I_{i-1}}(b_{i-1})$ \textbf{else} $v^l_i = \hat{f}_{S \cap I_{i+1}}(b_{i})$ \textbf{end if}
    \STATE \textbf{if} $b_{i-1} < 1$ \textbf{then } $v^u_i = \hat{f}_{S \cap I_{i+1}}(b_{i})$ \textbf{else} $v^u_i = \hat{f}_{S \cap I_{i-1}}(b_{i-1})$ \textbf{end if}
  \STATE \textbf{Output} $v^l_i + (x-b_{i-1})\frac{v^u_i-v^l_i}{b_i-b_{i-1}}$.\\
   \ENDIF
   \end{algorithmic}
  \end{algorithm}
\footnotetext{One can either think of the label as being fixed for every data point or if it is randomized, we need to use the same label for every datapoint once queried.}
Next, we state the number of label queries needed to make local predictions corresponding to the Query algorithm (Algorithm \ref{alg2}).
%
%

\begin{theorem}
\label{thm:lipschitzlocalquery}
For any distribution $\mathcal{D}$ over $[0,1]\times[0,1]$, Lipschitz constant $L>0$ and error parameter $\epsilon \in [0,1]$, let $(S, \p)$ be the output of (randomized) Algorithm~\ref{alg1} where $S$ is the set of unlabeled samples of size $O(\frac{L}{\epsilon^4}\log(\frac{1}{\epsilon}))$ and $\p$ is a partition of the domain $[0,1]$. Then, there exists a function  $\tilde{f} \in \mathcal{F}_L$, such that for all $x \in [0,1]$, Algorithm~\ref{alg2} queries $O(\frac{1}{\epsilon^4}\log(\frac{1}{\epsilon}))$ labels from the set $S$ and outputs $\text{Query}(x, S, P) $ satisfying
\begin{equation}
    \text{Query}(x, S, \p) = \tilde{f}(x)\;,
\end{equation}
and the function $\tilde{f}$ is $\epsilon$-optimal, that is, $ \diffd(\ft, \Fl) \leq \epsilon$ with probability greater than $\frac{1}{2}$.
\end{theorem}
\begin{proof}
 We begin by defining some notation. Let $S= \{x_i\}^M_{i=1}$ be the set of the unlabeled samples and $\p=\{[b_0,b_1],[b_1,b_2], \ldots\}$ be the partition returned by the pre-processing step given by Algorithm~\ref{alg1}. We will use $y_i$ to denote the queried label for the datapoint $x_i$. 
%
Let $\D_i$ be the distribution of a random variable  $(X, Y) \sim \D$ conditioned on the event $\{X \in I_i\}$. Similarly, let $\Dlg$ and $\Dsh$ be the conditional distribution of $\D$ on intervals belonging to $\plong$ and $\pshort$ respectively. Let $p_i$ denote the probability of a point sampled from distribution $\D$ lying in interval $I_i$. Going forward, we use the shorthand `probability of interval $I_i$' to denote $p_i$. Let $\plg$ and $\psh$ be the probability of set of long and short intervals respectively. Recall that $f_{\D}^*$ is the function which minimizes $\errd(f)$ for $f \in \Fl$ and $f_{\D_i}^*$ is the function which minimizes this error with respect to the conditional distribution $\D_i$. Let $M_i$ denote the number of unlabeled samples of $S$ lying in interval $I_i$. 
%
%
%
%
For any interval $I_i$, let $\hat{f}_{S\cap I_i}$ be the ERM with respect to that interval.
%
%
%
\paragraph{Lipschitzness of $\ft$.}
For any interval $I_i = [b_{i-1}, b_i]$, let $v_i^l = \hat{f}_{S\cap I_{i-1}}(b_{i-1})$ be the value of the function with minimum empirical error on the neighbouring interval $I_{i-1}$ at the boundary point $b_{i-1}$  and similarly $v_i^u = \hat{f}_{S\cap I_{i+1}}(b_{i})$ be the value of the function $\hat{f}_{S\cap I_{i+1}}$ at the boundary point $b_i$. Note that if $I_i$ is a boundary interval and therefor $I_{i-1}$ (or $I_{i+1}$) does not exist, we can define $v^l_i=v^u_i$ (or $v^u_i=v^l_i$).
Further, let $f^{\sf{int}}_i: I_i \mapsto \R$ be the linear function interpolating from $v_i^l$ to $v_i^u$, 
\begin{small}
\begin{equation*}
 f^{\sf{int}}_i(x) = v^l_i+(x-b_{i-1})\frac{v_i^u-v^l_i}{b_i-b_{i-1}}.
\end{equation*}
\end{small}
Note that the Query procedure (Algorithm~\ref{alg2}) is designed to output $\ft(x)$ for each query $x$ where
\begin{small}
\begin{equation*}
       \ft(x) = \begin{cases}
       \hat{f}_{S\cap I_i}(x) \quad &\text{ if } x \in I_i \text{ for any } I_i \in \plong\\
        f^{\sf{int}}_i(x) \quad &\text{ if } x \in I_i \text{ for any } I_i \in  \pshort
       \end{cases}.
\end{equation*}
\end{small}
%
%
%
%
%
Now, it is easy to see that the function $\ft(x)$ is $L$-Lipschitz. The function $\hat{f}_{S\cap I_i}$ on each of the long intervals is $L$-Lipschitz by construction. The function ${f_i^{\sf{int}}}$ on each of the short intervals is also $L$-Lipschitz since the short intervals in $\pshort$ have length ${1}/{L}$ and the label values  $y_j \in [0,1]$. Also, the function $\ft$ is continuous at the boundary of each interval by construction. 
%
\paragraph{Error Guarantees for $\ft$.}
Now looking at the error rate of the function $\tilde{f}(x)$ and following a repeated application of tower property of expectation, we get that 
\begin{small}
\begin{align}
    \diffd(\ft, \Fl) &= \plg\diff_{\Dlg}(\ft, f_{\D}^*) + \psh\diff_{\Dsh}(\ft, f_{\D}^*)\nonumber\\
    &= \sum_{i:I_i \in \plong}p_i\diff_{\D_i}(\ft, f_{\D}^*) + \psh\diff_{\Dsh}(\ft, f_{\D}^*)\label{eqn:total-error}
\end{align}
\end{small}
%
We now bound both terms above to obtain a bound on the total error of the function $\ft$.
%

\emph{Error for short intervals. } 
The probability of short intervals $\psh$ is small with high probability since the total length of short intervals is $\epsilon$ and the intervals are chosen uniformly randomly. More formally, from Lemma~\ref{lemma:prob_intervals}, we know that with probability at least $1-\delta$, the probability of short intervals $\psh$ is upper bounded by $\epsilon/\delta$. Also, the error for any function $f$ is bounded between $[0,1]$ since the function's range is $[0,1]$. Hence, we get that
\begin{align}
\psh\diff_{\D_{\pshort}}(\ft, f_{\D}^*) &\leq \frac{\epsilon}{\delta} \label{eqn:type2-error}
\end{align}
\emph{Error for long intervals:} We further divide the long intervals into 3 subtypes: 
{\small
\begin{align*}
    \begin{gathered}
    \plongone \defn \left\lbrace I_i \; |\; I_i \in \plong,\; p_i \geq \frac{1}{L},\; M_i \geq \frac{1}{2\epsilon^4}\log\left(\frac{1}{\epsilon}\right) \right\rbrace, 
    \plongtwo \defn \left\lbrace I_i \; |\; I_i \in \plong,\; p_i \geq \frac{1}{L},\; M_i < \frac{1}{2\epsilon^4}\log\left(\frac{1}{\epsilon}\right) \right\rbrace,\\ 
   \plongthree \defn \left\lbrace I_i \; |\; I_i \in \plong,\; p_i < \frac{1}{L} \right\rbrace.
    \end{gathered}
\end{align*}}
The intervals in both first and second subtypes have large probability $p_i$ with respect to distribution $\D$ but differ in the number of unlabeled samples in $S$ lying in them. Finally, the intervals in third subtype $\plongthree$ have small probability $p_i$ with respect to distribution $\D$. Now, we can divide the total error of long intervals into error in these subtypes
\begin{small}
\begin{align}
\sum_{i:I_i \in \plong}p_i\diff_{\D_i}(\ft, f_{\D}^*) &= \underbrace{\sum_{i:I_i \in \plongone}p_i\diff_{\D_i}(\ft, f_{\D}^*)}_{E1} + \underbrace{\sum_{i:I_i \in \plongtwo}p_i\diff_{\D_i}(\ft, f_{\D}^*)}_{E2}+ \underbrace{\sum_{i:I_i \in \plongthree}p_i\diff_{\D_i}(\ft, f_{\D}^*)}_{E3}.
\label{eqn:combined}
\end{align}
\end{small}
Now, we will argue about the contribution of each of the three terms above. 

\underline{Bounding $E3$.} Since there are at most $L\epsilon$ long intervals and each of these intervals $I_i$ has probability $p_i$ upper bounded by $1/L$, the total probability combined in these intervals is at most $\epsilon$. Also, in the worst case, the loss can be 1. Hence, we get an upper bound of $\epsilon$ on $E_3$.
%
%

\underline{Bounding $E2$.} From Lemma~\ref{lemma:number_unlabeled_samples}, we know that with failure probability at most $\delta$, these intervals have total probability upper bounded by ${\epsilon}/{\delta}$. Again, the loss can be 1 in the worst case. Hence, we can get an upper bound of ${\epsilon}/{\delta}$ on $E_2$.

\underline{Bounding $E1$.}  
Let $F_i$ denote the event that $\diff_{\D_i}(\hat{f} _{S\cap I_i}, f_{\D_i}^*) > \epsilon$.  The expected error of intervals $I_i$ in $\plongone$ is then
{\small
\begin{align*}
\E[\sum_{i:I_i \in \plongone}p_i\diff_{\D}(\ft, f_{\D}^*)] &\stackrel{\1}{\leq} \E[\sum_{i:I_i \in \plongone}p_i\diff_{\D_i}(\hat{f}_{S\cap I_i}, f_{\D_i}^*)]\\ 
&= \sum_{i:I_i \in \plongone}p_i(\E[\diff_{\D_i}(\hat{f}_{S\cap I_i}, f_{\D_i}^*)|F_i]\Pr(F_i) + \E[\diff_{\D_i}(\hat{f}_{S\cap I_i}, f_{\D_i}^*)|\neg F_i]\Pr(\neg F_i))\\
&\stackrel{\2}{\leq} \sum_{i:I_i \in \plongone}p_i(1\cdot \epsilon + \epsilon \cdot 1) 
\leq 2\epsilon,
\end{align*}}
where step $\1$ follows by noting that $\ft = \hat{f}_{S\cap I_i}$ for all long intervals $I_i \in \plong$ and that $f^*_{\D_i}(x)$ is the minimizer of the error $\err_{\D_i}(f)$ over all $L$-Lipschitz functions, 
and step $\2$ follows since $\E[\diff_{\D_i}(\hat{f}_{S\cap I_i}, f_{\D_i}^*)|\neg F_i] \leq \epsilon$ by the definition of event $F_i$ and $\Pr(F_i) \leq \epsilon$ follows from a standard uniform convergence argument (detailed in Lemma~\ref{lem: good type 1 intervals}).
Now, using Markov's inequality, we get that $E_1 \leq {2\epsilon}/{\delta}$ with failure probability at most $\delta$. 

Plugging the error bounds obtained in equations~\eqref{eqn:type2-error} and ~\eqref{eqn:combined} into equation~\eqref{eqn:total-error} and setting $\delta = \frac{1}{10}$ establishes the required claim.

%
%
%
%
%
%

\paragraph{Label Query Complexity.} For any query point $x^*$, $\ft(x^*)$ can be computed by  querying the labels of the interval in which $x^*$ lies (if $x^*$ lies in a long interval) or the two neighboring intervals of the interval in which $x^*$ lies (if $x^*$ lies in a short interval). Hence, the computation requires $O(({1}/{\epsilon^4})\log({1}/{\epsilon}))$ label queries over the set $S$ of $O(({L}/{\epsilon^4})\log({1}/{\epsilon}))$ unlabeled samples.
\end{proof}
\subsection{Guarantees for Error Estimation}
\label{subsection:lipschitz-1d-errorestimation}
We now study the Error Estimation problem for the class of Lipschitz functions $\Fl$. Our proposed estimator detailed in Algorithm~\ref{alg3} uses the algorithm for locally computing the labels of query points for a nearly optimal function from the previous section. In particular, it samples a few random query points and uses them to compute the average empirical error. The final query complexity of our procedure is then obtained via standard concentration arguments, relating the empirical error to the true expected error. We formalize this guarantee in  Theorem~\ref{thm:lipschitzerror} and defer its proof to Appendix~\ref{appendix:lipschitz-lowd}.
 \begin{algorithm}[H]
   \begin{algorithmic}[1]
\caption{Error$(L, \D, \epsilon)$\label{alg3}}
\STATE Let $S,\p = $Preprocess$(L, \D_x, \epsilon)$
   \STATE Sample a set $\{(x_1,y_1),(x_2,y_2),\cdots, (x_N,y_N)\}$ labeled examples from distribution $\mathcal{D}$ where $ N = O(\frac{1}{\epsilon^2})$ \\
   \STATE \textbf{Output} $\widehat{\errd}(\Fl) = \frac{1}{N}\sum_{i=1}^N|\text{Query}(x_i, S, \p)-y_i|$
   \end{algorithmic}
  \end{algorithm}
%
%
%
\begin{restatable}{theorem}{theoremlipschitzerror}
\label{thm:lipschitzerror}
For any distribution $\mathcal{D}$ over $[0,1]\times[0,1]$, Lipschitz constant $L>0$ and parameter $\epsilon \in [0,1]$, Algorithm~\ref{alg3} uses $O(\frac{1}{\epsilon^6}\log(\frac{1}{\epsilon}))$ active label queries on $O(\frac{L}{\epsilon^4}\log(\frac{1}{\epsilon}))$ unlabeled samples from distribution $\D_x$ and produces an output $\widehat{\errd}(\Fl)$ satisfying
{\small\begin{equation*}
|\widehat{\errd}(\Fl) - \errd( \Fl)| \leq \epsilon 
\end{equation*} }
with probability at least $\frac{1}{2}$.
\end{restatable}
\section{Nadaraya-Watson Estimator}
\label{section:kde}
In this section, we consider the related problem of approximating the minimum error that can be achieved by the Nadaraya-Watson estimator under a linear diagonal transformation with eigenvalues coming from a small range. In this setting, there exists a distribution $\mathcal{D}$ over the domain $\R^d$. Each data point $x\in\R^d$ has a true label $f(x) \in \{0,1\}$. The Nadaraya-Watson prediction algorithm when given a dataset $S=\{x_1,x_2,\cdots,x_N\}$ of unlabeled samples sampled from distribution $\mathcal{D}$ and a query point $x$ outputs the prediction 
\begin{small}
$$\tilde{f}_{S,K_A}(x) = \frac{\sum_{i=1}^NK_A(x_i,x)f(x_i)}{\sum_{i=1}^NK_A(x_i,x)} = \sum_{i=1}^Np_{S,A}(x_i,x)f(x_i)$$
\end{small}
where $K_A(x,y) = \nicefrac{1}{(1+||A(x-y)||_2^2)}$ is the kernel function for matrices $A \in \R^{d\times d}$. The loss of the data point $(x,f(x))$ with respect to the unlabeled samples $S$ and the kernel $K_A$ is $$l_{S,K_A}(x) = |f(x)-\tilde{f}_{S,K_A}(x)|$$
The total loss of the prediction function $\tilde{f}_{S,K_A}$ with respect to distribution $\mathcal{D}$ is 
{\small
$$L_{S,K_A} = E_{x\sim \mathcal{D}}|f(x)-\tilde{f}_{S,K_A}(x)|$$
}
Now, let us say we are interested in computing the prediction loss with the best
diagonal linear transformation $A$ for the data with
eigenvalues bounded between constants $1$ and $2$ that is a matrix $A 
\in \mathcal{A}$ where $\mathcal{A} = \{A \in \R^{d\times d} | A_{i,j} = 0 \  \forall i\neq j \text{ and } 1 \leq A_{i,i}\leq 2\}$ that is
{\small
$$L_{S} = \min_{A \in \mathcal{A}}L_{S,K_A} = \min_{A \in \mathcal{A}} \E_{x\sim
\mathcal{D}}\sum_ip_{S,A}(x_i,x)|f(x_i)-f(x)|$$
}
In this section, we show (Theorem~\ref{thm:kde-min-error}) that it is possible to
estimate the loss $L_S$ with an additive error of $\epsilon$ with  labeled sample complexity
$\tilde{O}({d}/{\epsilon^2})$ and running time  $\tilde{O}({d^2}/{\epsilon^{d+4}}+{dN}/{\epsilon^2})$. We use $[n]$ to denote $\{1,2,\cdots,n\}$ for a positive integer $n$ and $\tilde{O}$ notation to hide polylogarithmic factors in the input parameters and error rates.

First, we discuss a theorem from \cite{backurs2018efficient} which will be crucial in the proof of Theorem~\ref{thm:kde-min-error}. Theorem~\ref{thm:backurs2018efficient}, formally stated in the appendix, states that for certain nice kernels $K(x,y)$, it is possible to efficiently estimate $N^{-1}\sum_{i=1}^N K(x,x_i)$ with a multiplicative error of $\epsilon$ for any query $x$. As a direct corollary of Theorem~\ref{thm:backurs2018efficient}, we obtain that it is possible to efficiently estimate the probabilities $p_{S,A}(q,x_i)$, for all data points $x_i$ in $S$, queries $q \in \R^d$ and matrices $A$ in $\mathcal{A}$. Let us define $\hat{S}_{S,A}(q)$ to be the estimator for $\sum_{x_i\in S}K(q,x_i)$ as per Theorem~\ref{thm:backurs2018efficient}.
\begin{corollary}
\label{thm:backurs2018efficient-cor}
There exists a data structure that given a data set $S \subset \R^d$ with $|S| = N$, using $O(\frac{Nd}{\epsilon^2}\log(\frac{ N}{\delta}))$ space and preprocessing time, for any $A \in \mathcal{A}$ and a query $q \in \R^d$ and data point $x \in \R^d$, estimates $p_{S,A}(q, x) = \nicefrac{K_A(q,x)}{\sum_{y\in S}K_A(q,y)}$ by using the estimator $\hat{p}_{S,A}(q, x) = \nicefrac{K_A(q,x)}{\hat{S}_{S,A}(q)}$ with accuracy $(1\pm\epsilon)$ in time $O(\log(\frac{ N}{\delta})\frac{d}{\epsilon^2})$ with probability at least $1-\nicefrac{1}{poly(N)}-\delta$.
\end{corollary}
%
%
%
We state the algorithm NW Error (Algorithm~\ref{alg-kde}) for computing the minimum prediction error of the algorithm with respect to underlying distribution $\mathcal{D}$, the set of unlabeled training data $S$, the set of matrices $\mathcal{A}$, the error parameter $\epsilon$ and the failure probability $\delta$ and discuss the idea behind the algorithm in the next few paragraphs.

\paragraph{Naive Algorithm.} The naive algorithm would take $O({1}/{\epsilon^2})$ labeled data points sampled from distribution $\mathcal{D}$ for each of the matrices $A \in \mathcal{A}_{\epsilon}$, an $\epsilon$-cover of the set $\mathcal{A}$ (Lemma~\ref{lemma:kde-approximate}) of size $O({1}/{\epsilon^d})$ and compute the exact loss using the $N$ data points in the training set. Hence, the number of labels required in this algorithm is $N+O({1}/{\epsilon^{d+2}})$. 

\paragraph{Dependence on $N$ of label query complexity.} Using our algorithm, we achieve labeled sample complexity of $\tilde{O}({d}/{\epsilon^2})$,  independent of $N$ and depending only polynomially on $d$. For getting rid of the dependence on $N$, the idea is to first sample $O({1}/{\epsilon^2})$ samples from distribution $\mathcal{D}$ and then for each sample, sample a training data point from $S$ with probability proportional to $p_{S,A}$ for the matrix $A$. This gets rid of the dependence on $N$. However, we still have to repeat this procedure separately for every matrix $A \in \mathcal{A}$ which leads to a requirement of $O({1}/{\epsilon^{d+2}})$ labels.

\paragraph{Dependence on $d$ of label query complexity.}To eliminate the exponential dependence on $d$, we show that for matrices in $\mathcal{A}$, we can use importance sampling and the samples generated for the identity matrix $I$ suffice to estimate the loss for all matrices $A\in\mathcal{A}$ with appropriate scaling factors $p_{S,A}/p_{S,I}$. This is because the eigenvalues of all the matrices are bounded between constants and hence, the sampling probabilities $p_{S,A}$ are similar up to a multiplicative factor (Lemma~\ref{lemma:kde-ineq}). This leads to our desired labeled sample complexity of $\tilde{O}({d}/{\epsilon^2})$. The factor $d$ comes in because of using a union bound over all the exponential number of matrices in $\mathcal{A}_{\epsilon}$.

\paragraph{Running time.} However, using this approach directly, we obtain a running time of $\tilde{O}({Nd}/{\epsilon^{d+2}})$ because for each matrix $A$, for each sample, we have to compute ${p_{S,A}}/{p_{S,I}}$ which requires going over all the data points in the set $S$. To achieve better running times, we use the faster kernel density estimation algorithm \citep{backurs2018efficient} to compute approximate probabilities efficiently (Corollary~\ref{thm:backurs2018efficient-cor}) and obtain a running time of $\tilde{O}({d^2}/{\epsilon^{d+4}}+{Nd}/{\epsilon^2})$ separating the multiplicative dependence of $N$ and ${1}/{\epsilon^d}$. We state the formal guarantees in Theorem~\ref{thm:kde-min-error} and its proof in Appendix~\ref{appendix:kde}. 
 \begin{algorithm}[t]
   \begin{algorithmic}[1]
\caption{NW Error$(S, \mathcal{A}, \mathcal{D}, \epsilon, \delta)$\label{alg-kde}}
    \STATE Let $\mathcal{A}_{\epsilon}=\{A \in \R^{d\times d}\ |\ A_{i,j} = 0 \  \forall\ i\neq j \text{ and } A_{i,i}\in\{1,1+\epsilon,(1+\epsilon)^2,\cdots,2\}$ $\forall i \in [d]\}$. 
   \STATE Sample $M = O\left(\frac{1}{\epsilon^2}\log\left(\frac{|\mathcal{A}_{\epsilon}|}{\delta}\right)\right)$ labeled examples $\{(z_i, f(z_i)\}_{i=1}^M$ with each $(z_i, f(z_i)) \sim \D$. \\
   \FOR{$i=1$ to $M$}
   \STATE Sample a $\zt_i$ with probability proportional to $p_{S,I}(z_i, \zt_i)$. 
   \ENDFOR
      \FOR{$A \in \mathcal{A}_{\epsilon}$}
         \FOR{$i=1$ to $M$}
            \STATE Compute $\hat{p}_{S,A}(z_i,\zt_i) = \frac{K_A(z_i,\zt_i)}{\hat{S}_{S,A}(z_i)}$.
        \ENDFOR
   \STATE  Compute $\hat{L}_{S,K_A} = \frac{1}{M}\sum_{i=1}^M|f(z_i)-f(\zt_i)|\frac{\hat{p}_{S,A}(z_i,\zt_i)}{p_{S,I}(z_i,\zt_i)}$.
      \ENDFOR
      \STATE \textbf{Output} $\hat{L}_S = \min_{A \in \mathcal{A}_{\epsilon}}\hat{L}_{S,K_A}$.
   \end{algorithmic}
  \end{algorithm}
\begin{restatable}{theorem}{theoremkdeminerror}
\label{thm:kde-min-error}
\sloppy For a $d$-dimensional unlabeled pointset $S$ with $|S| = N$, Algorithm~\ref{alg-kde} queries $O(\frac{1}{\epsilon^2}(d\log(\frac{1}{\epsilon})+\log(\frac{1}{\delta}))$ labels from $S$ and outputs $\hat{L}_S$ such that
\begin{small}
\begin{equation*}
    |\hat{L}_S - \min_{A \in \mathcal{A}} \E_{x\sim
\mathcal{D}}\sum_ip_{S,A}(x_i,x)|f(x_i)-f(x)| | \leq \epsilon
\end{equation*}
\end{small}
with a failure probability of at most $\delta + \frac{d}{\epsilon^{d+2} poly(N)}\log(\frac{1}{\epsilon\delta})$ and runs in time
$\tilde{O}(\frac{d^2}{\epsilon^{d+4}} + \frac{dN}{\epsilon^2})$. 
\end{restatable}
%
%
%
%
\section{Conclusion}
\label{section:conclusion}
We gave an algorithm to approximate the optimal prediction error for the class of bounded $L$-Lipschitz functions with independent of $L$ label queries. We also established that for any given query point, we can estimate the value of a nearly optimal function at the query point locally with label queries, independent of $L$. It would be interesting to extend these notions of error prediction and local prediction to other function classes. Finally, we also gave an algorithm to approximate the minimum error of the Nadaraya-Watson prediction rule under a linear diagonal transformation with eigenvalues in a small range which is both sample and time efficient.
\acks{We thank the anonymous reviewers who helped improve the readability and presentation of this draft by providing many helpful comments. This work was done in part while Neha Gupta was visiting  TTIC. This work was supported in part by the National Science Foundation under grant CCF-1815011.}
\bibliography{references}
\newpage
\appendix
\section{One Dimensional Case}
\label{appendix:lipschitz-lowd}
We first re-state Theorem~\ref{thm:lipschitzerror} which gives sample complexity guarantees for error estimation for the class of one-dimensional Lipschitz functions and also give the proof. 
\theoremlipschitzerror*
\begin{proof}
 By Theorem~\ref{thm:lipschitzlocalquery}, we know that $\text{Query}(x, S, \p) = \ft(x)\ \forall x \in [0,1]$ and 
 the error of $\ft$ additively approximates the error of function $f_{\D}^*$, that is, 
 \begin{equation*}
      \diffd(\ft, \Fl) = \errd(\ft) - \errd(\Fl) \leq \epsilon
 \end{equation*}
 with probability greater than $1/2$. Thus, with probability greater than $1/2$, we get 
 \begin{align*}
     |\widehat{\errd}(\Fl) - \errd(\ft)| &= |\frac{1}{N}\sum_{i=1}^N|\text{Query}(x_i, S, \p)-y_i| - \errd(\ft)|\\
     &= |\frac{1}{N}\sum_{i=1}^N|\ft(x_i)-y_i| - \errd(\ft)| \leq \epsilon
 \end{align*}
 The last inequality follows by standard concentration arguments since $N \geq O({1}/{\epsilon^2})$. The theorem statement follows by using triangle inequality.

By Theorem~\ref{thm:lipschitzlocalquery}, the number of unlabeled samples is $O(({L}/{\epsilon^4})\log({1}/{\epsilon}))$ and the number of label queries is $ O(({1}/{\epsilon^2})\cdot(({1}/{\epsilon^4})\log({1}/{\epsilon}))) = O(({1}/{\epsilon^6})\log({1}/{\epsilon}))$.
\end{proof}
Now, we will state the lemmas involved in the proof of Theorem~\ref{thm:lipschitzlocalquery} with their proofs. 
The following lemma proves that with enough unlabeled samples, a large fraction of long intervals have enough unlabeled samples in them which is eventually used to argue that they will be sufficient to learn a function which is approximately close to the optimal Lipschitz function over that interval. 
\begin{lemma}
\label{lemma:number_unlabeled_samples}
For any distribution $\D_x$, consider a set $S = \{x_1, x_2, \cdots, x_M\}$ of unlabeled samples where each sample $x_i\stackrel{\text{i.i.d.}}{\sim}\D_x$. Let $\mathcal{G}$ be the set of long intervals $\{I_i\}$ each of which satisfies $p_i = \Pr_{x \sim \D_x}(x \in I_i) \geq \frac{1}{L}$. Let $E_i$ denote the event that $\sum_{x_j \in S}\mathbb{I}[x_j \in I_i] < \frac{1}{2\epsilon^4}\log(\frac{1}{\epsilon})$. Then, we have 
\begin{equation*}
    \sum_{I_i \in \mathcal{G}}p_i\mathbb{I}[E_i] \leq \frac{\epsilon}{\delta}
\end{equation*}
with failure probability atmost $\delta$ for $M = \Omega(\frac{L}{\epsilon^4}\log(\frac{1}{\epsilon}))$. 
\end{lemma}
\begin{proof}
For any interval $I \in \mathcal{G}$, we have that $\E[\sum_{x_j \in S}\mathbb{I}[x_j \in I]] \geq \frac{1}{\epsilon^4}\log(\frac{1}{\epsilon})$. Using Hoeffding inequality, we can get that $\Pr(E_i) \leq \epsilon$ for all intervals $I_i \in \mathcal{G}$. Calculating expectation of the desired quantity, we get
\begin{equation*}
    \E[\sum_{I_i \in \mathcal{G}}p_i\mathbb{I}[E_i]] = \sum_{I_i \in \mathcal{G}}p_i\Pr[E_i] \leq \epsilon\sum_{I_i \in \mathcal{G}}p_i \leq \epsilon
\end{equation*}
%
We get the desired result using Markov's inequality.
\end{proof}
The following lemma states that the probability of short intervals $\psh$ is small with high probability. Consider the case of uniform distributions. In this case, since the short intervals cover only $\epsilon$ fraction of the $[0,1]$ length, their probability $\psh$ is upper bounded by $\epsilon$. The case for arbitrary distributions holds because the intervals are chosen randomly.
\begin{lemma}
\label{lemma:prob_intervals}
When we divide the $[0,1]$ domain into alternating intervals of length $\frac{1}{L\epsilon}$ and $\frac{1}{L}$ with a random offset at $\{0, 1, 2, \cdots, \frac{1}{\epsilon}\}\frac{1}{L}$ as in the preprocessing step, then
\begin{equation*}
    \psh = \sum_{I_i \in \pshort}p_i \leq \frac{\epsilon}{\delta}
\end{equation*}
with failure probability atmost $\delta$. 
\end{lemma}
\begin{proof}
Now, if we consider the division of $[0,1]$ into alternating intervals of length ${1}/(L\epsilon)$ and ${1}/{L}$ with the offset chosen uniformly randomly from $\{0, 1, 2, \cdots, {1}/{\epsilon}\}({1}/{L})$, then the intervals of length ${1}/{L}$ combined are disjoint in each of these divisions and together cover the entire $[0,1]$ length. Hence, there are at most $\delta$ fraction out of the total ${1}/{\epsilon}$ cases where the short intervals have probability greater than ${\epsilon}/{\delta}$. Hence, with probability $1-\delta$, the short intervals have probability upper bounded by ${\epsilon}/{\delta}$. 
\end{proof}
\begin{lemma}
\label{lem: good type 1 intervals}
Let $I_i \in \plongone$ be any long interval of subtype 1. For the event \mbox{$F_i = \{\diff_{\D_i}(\hat{f}_{S \cap I_i}, f_{\D_i}^*) > \epsilon
\}$}, we have
\begin{equation*}
    \Pr(F_i) \leq \epsilon
\end{equation*}
\end{lemma}

\begin{proof}
We know that the covering number for the class of one dimensional $L$-Lipschitz functions supported on the interval $[0,l]$ is $O({Ll}/{\epsilon})$ (Lemma~\ref{lemma:lipschitz_covering}). For, Lipschitz functions supported on a long interval of length $l={1}/{(L\epsilon)}$, we get this complexity as $O({1}/{\epsilon^2})$.  We know by standard results in uniform convergence, that the number of samples required for uniform convergence up to an error of $\epsilon$ and failure probability $\delta$ for all functions in a class $\mathcal{F}_L$ is $O(((\text{Covering number of } \mathcal{F}_L)/\epsilon^2)\log({1}/{\delta}))$ (Lemma~\ref{lemma:lipschitz_sample_complexity}) and hence, for long intervals we get this estimate as $O(({1}/{\epsilon^4})\log(1/\delta))$. 
Given that there are $\Omega(({1}/{\epsilon^4})\log(1/\epsilon))$ samples in these intervals by design, the function learned using these samples is only additively worse than the function with minimum error using standard learning theory arguments with failure probability at most $\epsilon$. Thus, we get that 
$$\diff_{\D_i}(f_{S \cap I_i}, f_{\D_i}^*) \leq \epsilon \quad \forall I_i \in \plongone$$
with failure probability at most $\epsilon$. 
\end{proof}
Now, we state the definitions of covering number of a metric space and the uniform covering number of a hypothesis class. These definitions are used in Lemma~\ref{lemma:lipschitz_sample_complexity} to argue how fast the empirical error converges to the expected error for a given hypothesis class.
\begin{definition}
\label{defn:covering}
$N(\epsilon,A,\rho)$ is the covering number of the metric space $A$ with respect to distance measure $\rho$ at scale $\epsilon$ and is defined as
\begin{equation*}
    N(\epsilon,A,\rho) = \min\{|C|\ |\ C \text{ is an } \epsilon \text{-cover of } A \text{ wrt } \rho\ (\forall x \in A, \exists c \in C \text{ st } \rho(c,x) \leq \epsilon)\}. 
\end{equation*}
%
\end{definition}
\begin{definition}
\label{defn:uniform_covering}
$N_p(\epsilon, F, m)$ for $p \in \{1,2,\infty\}$ is  the uniform covering number of the hypothesis class $F$ at scale $\epsilon$ with respect to distance measure $d_p$ where $d_p(x,y) = ||x-y||_p$ and is defined as 
\begin{equation*}
    N_p(\epsilon, F, m) \defn \max_{x \in X^m}N(\epsilon,\{[f(x_1), f(x_2), \cdots, f(x_m)]\}_{f \in F},d_p)
\end{equation*}
where $N(\epsilon,A,\rho)$ is as defined in definition~\ref{defn:covering}.
\end{definition}
The following lemma uses the well known generalization theory to argue how fast the empirical error uniformly converges to the expected error for the class of $d$-dimensional Lipschitz functions. 
\begin{lemma}
\label{lemma:lipschitz_sample_complexity}
\sloppy Let $S = \{(x_i,y_i)\}^M_{i=1}$ be a set of $M > \frac{1}{\epsilon^2}(\Omega(\frac{Ll}{\epsilon}))^d\log(\frac{1}{\delta})$ data points sampled uniformly randomly from distribution $\D$. Then, 
\begin{equation*}
\errd(\hat{f}_S) - \errd(\Fl) \leq \epsilon
\end{equation*}
with probability at least $1-\delta$. 
\end{lemma}
\begin{proof}
A standard result for uniform convergence for general loss functions (For example, Theorem 21.1 in \cite{anthony2009neural}) states that
\begin{equation*}
\Pr_{S\sim \mathcal{D}^n}[\sup_{f\in F}|er_\mathcal{D}^l[f]-er^l_S[f]|]> \epsilon] \leq 4N_1\left(\frac{\epsilon}{8}, l_F, 2m\right)e^{\frac{-m\epsilon^2}{32}}
\end{equation*}
where $l$ is the loss function bounded between $[0,1]$ and $F$ is a class of functions mapping into $[0,1]$. 

We will prove that $N_1\left(\frac{\epsilon}{8}, l_F, 2m\right) \leq (O(\frac{Ll}{\epsilon}))^d$ which will complete the proof of the lemma by using triangle inequality and uniform convergence for $f_{\D}^*$ and $\hat{f}_S$.
\begin{equation*}
    \log\left(N_1\left(\frac{\epsilon}{8}, l_F, 2m\right)\right) \stackrel{\1}{\leq} \log\left(N_1\left(\frac{\epsilon}{8}, F, 2m\right)\right) \stackrel{\2}{\leq} \log\left(N_\infty\left(\frac{\epsilon}{8}, F, 2m\right)\right)\\
    \stackrel{\3}{\leq} \left(O\left(\frac{Ll}{\epsilon}\right)\right)^d
\end{equation*}
$\1$ follows because $|l(f(x_i), y_i)-l(f'(x_i),y_i)| \leq |f(x_i)-f'(x_i)|$ for all  $f,f' \in F$. $\2$ follows from Lemma 10.5 in \cite{anthony2009neural} and $\3$ follows from Lemma~\ref{lemma:lipschitz_covering}. 
\end{proof}
Now, we state the definition of Lipschitz extension in Theorem~\ref{thm:lipschitz-extension} and the result which states that for any metric space, if we have a $L$-Lipschitz function on a subset of the metric space, then it is possible to extend the function to the entire space with respect to the metric which preserves the values of the function at the points originally in the domain and is now $L$-Lipschitz on the entire domain. This will be used in the computing covering number bounds for the class of Lipschitz functions (Lemma~\ref{lemma:lipschitz_covering}).
\begin{theorem}
\label{thm:lipschitz-extension}
[Theorem 1 from \cite{mcshane1934extension}]
For a $L$-Lipschitz function $f:E\rightarrow \R$ defined on a subset $E$ of the metric space $S$, $f$ can be extended to $S$ such that its values on the subset $E$ is preserved and it satisfies the $L$-Lipschitz property over the entire domain $S$ with respect to the same metric. Such an extension is called Lipschitz extension of the function $f$. 
\end{theorem}
Now, we will state the well known covering number bounds for the class of high dimensional Lipschitz functions. Note that we have stated this here just for completeness and the proof essentially follows the proof from \cite{gottlieb2017efficient}. 
\begin{lemma}
\label{lemma:lipschitz_covering}
For the class of high dimensional Lipschitz functions $\mathcal{F}_L:[0,l]^d \rightarrow [0,1]$ where $f \in \mathcal{F}_L$ satisfies $|f(x)-f(y)| \leq L||x-y||_\infty \  \forall x,y \in [0,l]^d$, we have $\log(N_\infty(\epsilon, F, m)) = (O(\frac{Ll}{\epsilon}))^d$. 
\end{lemma}
\begin{proof}
Let us consider a discretization of the domain $P=[0,l]^d$ where we divide each coordinate of the domain into intervals of length $\frac{\epsilon}{3L}$. Let us consider a set $\mathcal{F}_{L}^\epsilon$ of all those Lipschitz functions which are the Lipschitz extensions of the Lipschitz functions which have output values amongst $R = \{0,\frac{\epsilon}{3},\frac{2\epsilon}{3},\cdots,1\}$ at the discretized points of the domain $P$ (note that it is always possible to form a Lipschitz extension of a Lipschitz function over metric space  by \cite{mcshane1934extension}, also mentioned in Theorem~\ref{thm:lipschitz-extension} above). So, we have that $|\mathcal{F}_{L}^{\epsilon}| \leq (\frac{3}{\epsilon})^{(\frac{3Ll}{\epsilon})^d}$.

Now, we will show that $\mathcal{F}_L^{\epsilon}$ forms a valid covering of the function class $\mathcal{F}_L$, and hence we get
$\log(N_\infty(\epsilon, \mathcal{F}_L, m)) = (O(\frac{Ll}{\epsilon}))^d$. Now, let us show that for any $f \in \mathcal{F}_L$, there exists a function $\tilde{f} \in \mathcal{F}_{L}^{\epsilon}$ such that $\sup_x|f(x)-\tilde{f}(x)| \leq \epsilon$. 

Consider a function $\hat{f}$ such that $\hat{f}(x) = \argmin_{y \in R}|y-f(x)|$ at the discretization of the domain $P$ and let $\tilde{f}$ be its Lipschitz extension. First, we will argue that $\hat{f}(x)$ is $L$-Lipschitz and since $\tilde{f}$ is a Lipschitz extension of $\hat{f}$, $\tilde{f}$ is also $L$-Lipschitz and by construction belongs to $F_L^{\epsilon}$. 

Now, we will prove that $\hat{f}$ is $L$-Lipschitz restricted to the discretization of the domain. Consider any $x,y$ in the discretized domain with $||x-y||_\infty \leq \frac{\epsilon}{3L}$, we have $|\hat{f}(x)-\hat{f}(y)| \leq L||x-y||_{\infty}$ because if the smaller value say $f(y)$ gets rounded down and the larger value $f(x)$ gets rounded above, this would violate the $L$-Lipschitzness of the function $f$. Hence, we see that $\hat{f}$ is $L$-Lipschitz. 

Now, we will show that  $\sup_x|f(x)-\tilde{f}(x)| \leq \epsilon$. Consider any point $x \in [0,l]^d$. Now, we know that there exists a $\tilde{x} \in P$ i.e. in the discretization of the domain such that $||x-\tilde{x}||_\infty \leq \frac{\epsilon}{3L}$. Hence, we get that $|\tilde{f}(x)-f(x)| \leq |\tilde{f}(x)-\tilde{f}(\tilde{x})| + |\tilde{f}(\tilde{x})-f(\tilde{x})| + |f(\tilde{x})-f(x)| \leq \frac{\epsilon}{3} + \frac{\epsilon}{3} + \frac{\epsilon}{3} \leq \epsilon$ since the functions $\tilde{f}$ and $f$ are both $L$-Lipschitz and $|\tilde{f}(\tilde{x})-f(\tilde{x})| = |\hat{f}(\tilde{x})-f(\tilde{x})| \leq \frac{\epsilon}{3}$. 
\end{proof}
%
%
\section{High Dimensional Case}
\label{subsection:lipschitz-highd}
\subsection{Problem Setup}
We recall the setting from the one dimensional case.
For any fixed $L>0$, let $\mathcal{F}_{L}$ be the class of $d$ dimensional functions supported on the domain $[0,1]^d$ with Lipschitz constant at most $L$, that is,
\begin{equation}
    \mathcal{F}_L = \{f:[0,1]^d\mapsto [0,1], \  |f(x)-f(y)|\leq L \|x-y\|_{\infty} \;  \forall x,y \in [0,1]^d\}.
\end{equation}
 Let $\epsilon$ be the error parameter. We will think of the dimension $d$ as constant with respect to $L$ and $\epsilon$.
Like the one-dimensional case, our algorithm for local predictions first involves a preprocessing step (Algorithm~\ref{alg4}) 
which takes as input the Lipschitz constant $L$, sampling access to distribution $\D_x$ and the error parameter $\epsilon$ and returns a partition $\p$ where the partition along dimension $j$ is $\p^j = \{ [b^j_0, b^j_1],[b^j_1, b^j_2], \ldots, \}$ and a set $S$ of unlabeled samples. The partition $\p$ consists of alternating intervals\footnote{Note that long intervals at the boundary could be shorter, but those can be handled similarly. } of length ${2}/{L}$ and ${d}/(L\epsilon)$ along each dimension. Let us divide these intervals along each dimension $j$ further into the two sets 
\begin{equation*}
    \begin{gathered}
    \plong^j\defn\{[b^j_0,b^j_1],[b^j_2,b^j_3],\ldots,\}\quad \text{(long intervals for dimension } j \text{ of length } {d}/(L\epsilon)), \\
    \pshort^j\defn\{[b^j_1,b^j_2],[b^j_3,b^j_4],\ldots,\}\quad \text{(short intervals for dimension } j \text{ of length }{2}/L).\\
    \end{gathered}
\end{equation*}

A data point $x$ is said to belong to a box $B_J$ (where $J$ is a vector of dimension $d$) if $x$ belongs to interval $B_J^j = I^j_{J_j} = [b^j_{J_j-1}, b^j_{J_j}]$ along the $j$th dimension. A data point $x$ which belongs to a short interval in $\pshort^j$ along at least one of the dimensions $j$ is said to belong to a set of short boxes $\pshort$ and otherwise, is said to belong to a set of long boxes $\plong$.
 \begin{algorithm}[H]
   \begin{algorithmic}[1]
\caption{Preprocess($L, \D_x, \epsilon$)\label{alg4}}
   \STATE Sample a uniformly random offset $b^j_1$ from $\{0,1,2,\cdots,\frac{d}{2\epsilon}\}\frac{2}{L}$ for each dimension $j \in [d]$. 
   \STATE Divide the $[0,1]$  interval along each dimension $j$ into alternating intervals of length $\frac{d}{L\epsilon}$ and $\frac{2}{L}$ with boundary at $b^j_1$ and let $\p$ be the resulting partition, that is, $\p^j = \{[b^j_0=0,b^j_1],[b^j_1, b^j_2], \ldots, ], \cdots, 1\}$ where $b^j_2 = b^j_1 + \frac{2}{L}, b_3^j = b_2^j + \frac{d}{L\epsilon}, \ldots, $. 
   \STATE Sample a set $S = \{x_i\}_{i=1}^M$ of $M = O((\frac{L}{\epsilon})^d \frac{1}{\epsilon^3}\log(\frac{1}{\epsilon}))$ unlabeled examples from distribution $\mathcal{D}_x$.\\
   \STATE \textbf{Output} $S, \p$.
   \end{algorithmic}
  \end{algorithm}
We give the definition of the extension box for a long box in $\plong$. It consists of a box interval and a cuboidal shell of thickness $\frac{1}{L}$ around it in each dimension with cut off at the boundary.
\begin{definition}
\label{defn:extension-interval}
For any long box $B_J$ where $B_J^i = I^i_{J_i}=[b^i_{J_i-1}, b^i_{J_i}]$, the extension box is defined to be $\hat{B}_J$ where $\hat{B}_J^i = \hat{I}_{J_i}^i = [\max(0,b^i_{J_i-1}-\frac{1}{L}),\min(b^i_{J_i}+\frac{1}{L},1)]$ $\forall i \in [d]$ such that $B_J \subset \hat{B}_J$.\end{definition}
In other words, box $\hat{B}_J$ consists of the box $B_J$ and a cuboidal shell of thickness $\frac{1}{L}$ around it on both sides in each dimension unless it does not extend beyond $[0,1]$ in each dimension.

The Query algorithm (Algorithm~\ref{alg5}) for test point $x^*$ takes as input the set $S$ of unlabeled examples and the partition $\p$ returned by the Preprocess algorithm. Note that all subsequent queries use the same partition $\p$ and the same set of unlabeled examples $S$. The algorithm uses different learning strategies depending on whether $x^*$ belongs to one of the long boxes in $\plong$ or short boxes in $\pshort$.

For the long boxes, it outputs the prediction corresponding to the empirical risk minimizer (ERM) function restricted to that box. If a query point $x$ lies in a short box which is the Lipschitz extension of a given long box, the function learned is the Lipschitz extension of function over the long box. The middle function value of the short box along each dimension is constrained to be $1$ which makes the overall function Lipschitz. We bound the expected prediction error of this scheme with respect to class $\Fl$ by separately bounding this error for long and short boxes. For the long boxes, we prove that the ERM has low error by ensuring that each box contains enough unlabeled samples. On the other hand, we show that the short boxes do not contribute much to the error because of their low probability under the distribution $\D$.
 \begin{algorithm}[H]
   \begin{algorithmic}[1]
\caption{Query($x, S, \p=[\{[b^j_0,b^j_1],[b^j_1, b^j_2],\cdots\}]^d_{j=1}$)\label{alg5}}
   \IF{query $x \in B_J \text{ where } B_J^j = I^j_{J_j} = [b^j_{J_j-1}, b^j_{J_j}] \text{ when } B_J \in \plong$ }
     \STATE Query labels for $x \in S \cap B_J$
   \STATE \textbf{Output} $\hat{f}_{S \cap B_J}(x)$ \\

    \ELSIF{query $x \in \hat{B}_J$ where  $\hat{B}_J$ is the extension of a long box $B_J$ where $B^j_{J} = I^j_{J_j} = [b^j_{J_j-1}, b^j_{J_j}]$ is a long interval}
     \STATE Query labels for $x \in S \cap B_J$
   \STATE \textbf{Output:} $f(x)$ where $f(x)$ is the Lipschitz extension of $\hat{f}_{S\cap B_J}(x)$ to the extension set $\hat{B}_{J}$ with constraints $f(x) = 1$ $\forall x \in \hat{B}_{J}$ satisfying $x_i \in \Big\{\frac{b^i_{J_i}+b^i_{J_i+1}}{2},\frac{b^i_{J_i-2}+b^i_{J_i-1}}{2}\Big\}$ for any dimension $i \in [d]$. 
\ELSE
    \STATE \textbf{Output:} 1
\ENDIF
   \end{algorithmic}
  \end{algorithm}
Next, we state the number of label queries needed to making local predictions corresponding to the Query algorithm (Algorithm \ref{alg5}).
\begin{theorem}
\label{thm:lipschitzlocalquery-highd}
For any distribution $\mathcal{D}$ over $[0,1]^d\times[0,1]$, Lipschitz constant $L>0$ and error parameter $\epsilon \in [0,1]$, let $(S, P)$ be the output of (randomized) Algorithm~\ref{alg1} where $S$ is the set of unlabeled samples of size $(O(\frac{L}{\epsilon}))^d\frac{1}{\epsilon^3}\log(\frac{1}{\epsilon})$ and $P$ is a partition of the domain $[0,1]^d$. Then, there exists a function  $\tilde{f} \in \mathcal{F}_L$, such that for all $x \in [0,1]$, Algorithm~\ref{alg5} queries $\frac{1}{\epsilon^2}(O(\frac{d}{\epsilon^2}))^d\log(\frac{1}{\epsilon})$ labels from the set $S$ and outputs $\text{Query}(x, S, P) $ satisfying
\begin{equation}
    \text{Query}(x, S, P) = \tilde{f}(x)\;,
\end{equation}
and the function $\tilde{f}$ is $\epsilon$-optimal, that is, $ \diffd(\ft, \Fl) \leq \epsilon$ with probability greater than $\frac{1}{2}$.
\end{theorem}
%
\begin{proof}
 We begin by defining some notation. Let $S= \{x_i\}^M_{i=1}$ be the set of the unlabeled samples and $\p=[\{[b^j_0,b^j_1],[b^j_1,b^j_2], \ldots\}]_{j=1}^d$ be the partition returned by the pre-processing step given by Algorithm~\ref{alg4}. 
Let $\D_J$ be the distribution of a random variable  $(X, Y) \sim \D$ conditioned on the event $\{X \in B_J\}$. Similarly, let $\Dsh$ and $\Dlg$ be the conditional distribution of $\D$ on boxes belonging to $\plong$ and $\pshort$ respectively. Let $p_J$ denote the probability of a point sampled from distribution $\D$ lying in box $B_J$. Going forward, we use the shorthand `probability of box $B_J$' to denote $p_J$. Let $\plg$ and $\psh$ be the probability of set of long and short boxes respectively. Recall that $f_{\D}^*$ is the function which minimizes $\errd(f)$ for $f \in \Fl$ and $f_{\D_J}^*$ is the function which minimizes this error with respect to the conditional distribution $\D_J$. Let $M_J$ denote the number of unlabeled samples of $S$ lying in box $B_J$. For any box $B_J$, let $\hat{f}_{S\cap B_J}$ be the ERM with respect to that interval.%
\paragraph{Lipschitzness of $\ft$.} 
For any long box $B_J$, we can define the extension function $f^{\sf{ext}}_J: \hat{B}_J \mapsto \R$ as the $L$-Lipschitz extension of the function $\hat{f}_{S\cap B_J}:B_J\mapsto \R$ to the extension interval $\hat{B}_J$ restricted to $1$ at the boundaries of the extension interval $\hat{B}_J$. 
The Query procedure (Algorithm~\ref{alg5}) is designed to output $\ft(x)$ for each query $x$ where
\begin{small}
\begin{equation*}
       \ft(x) = \begin{cases}
       \hat{f}_{S\cap B_J}(x) \quad &\text{ if } x \in B_J \text{ for any } B_J \in \plong\\
        f^{\sf{ext}}_J(x) \quad &\text{ else if } x \in \hat{B}_J \text{ for any } B_J \in  \plong\\
        1 \quad &\text{ otherwise }
       \end{cases}.
\end{equation*}
\end{small}
%
%
%
%
Now, we will argue that $\tilde{f}(x)$ is $L$-Lipschitz for all $x\in[0,1]^d$. The function $\tilde{f}(x)$ for each long box is $L$-Lipschitz by construction. Each of the extension functions $f^{\sf{ext}}_J$ is also $L$-Lipschitz by construction if it exists. We need to prove that such a Lipschitz extension exists. By \cite{mcshane1934extension} (also stated as Theorem~\ref{thm:lipschitz-extension} in this paper), we can see that such a Lipschitz extension always exists because for point $x \in B_J$ and $y$ belonging to the boundary of the extension interval $\hat{B}_J$, we get that $||x-y||_{\infty} \geq \frac{1}{L}$ and $|f(x)-f(y)|\leq 1$.  For query points $x$ which do not belong to any of these extension intervals, we output $1$. Basically, these points belong to the short intervals at the boundary of the domain and equivalent to learning Lipschitz extension with empty long interval and constrained to be $1$ at the boundary. Hence, the function takes $1$ everywhere in this interval. This shows that each of the functions is individually Lipschitz. 

Now, we will argue that the function is also continuous. 
For each long box $B_J$ where $I^j_J = [b^j_{J^j-1},b^j_{J^j}]$ $\forall j \in [d]$, consider the extension interval $\hat{I}_J$ as defined in Definition~\ref{defn:extension-interval}. The function $\tilde{f}(x)$ is constrained to be $1$ for the middle point of the short interval along each dimension i.e. $\tilde{f}(x) = 1$ $\forall x\in[0,1]^d\text{ if } \exists i\in[d] \text{ such that }x_i=\frac{b^i_{j-1}+b^i_{j}}{2}$ where $[b^i_{j-1},b^i_{j}]$ is a short interval for dimension $i$. Note that these middle points of short boxes are precisely the only points of intersection between extensions of two different long boxes because long intervals in each dimension are separated by a distance of $\frac{2}{L}$ by construction and the extension boxes extend up to a shell of thickness around $\frac{1}{L}$ in each dimension. Now, for a query point $x$ belonging to extension interval $\hat{I}_J$ for a long box $B_J$ takes the value $f(x)$ where $f(x)$ is the Lipschitz extension of the Lipschitz function $f_J$ to the superset $\hat{B}_J$ and constrained to be $1$ at the boundary of the shell. Hence, we can see that the function $\tilde{f}$ is $L$-Lipschitz. 
\paragraph{Error Guarantees for $\ft$.} Now looking at the error rate of the function $\tilde{f}(x)$ and following a repeated application of tower property of expectation, we get that 
\begin{align}
    \diffd(\ft, \Fl) &= \plg\diff_{\Dlg}(\ft, f_{\D}^*) + \psh\diff_{\Dsh}(\ft, f_{\D}^*)\nonumber\\
    &= \sum_{J:B_J \in \plong}p_J\diff_{\D_J}(\ft, f_{\D}^*) + \psh\diff_{\Dsh}(\ft, f_{\D}^*)\label{eqn:total-error-highd}
\end{align}
Now, we need to argue about the error bounds for both long and short boxes to argue about the total error of the function $\tilde{f}$.

\emph{Error for short boxes. } 
From Lemma~\ref{lemma:prob_intervals-highd}, we know that with probability at least $1-\delta$, the probability of short boxes $\psh$ is upper bounded by $2\epsilon/\delta$. Also, the error for any function $f$ is bounded between $[0,1]$ since the function's range is $[0,1]$. Hence, we get that
\begin{align}
\psh\diff_{\D_{\pshort}}(\ft, f_{\D}^*) &\leq \frac{2\epsilon}{\delta} \label{eqn:type2-error-highd}
\end{align}
\emph{Error for long boxes:} We further divide the long boxes into 3 subtypes: 
{\small
\begin{align*}
    \plongone &\defn \left\lbrace B_J \; |\; B_J \in \plong,\; p_J \geq\frac{\epsilon}{(\frac{L\epsilon}{d})^d},\; M_J \geq  \frac{c_d}{\epsilon^2}\left(\frac{d}{\epsilon^2}\right)^d\log\left(\frac{1}{\epsilon}\right) \right\rbrace,\\ 
    \plongtwo &\defn \left\lbrace B_J \; |\; B_J \in \plong,\; p_J \geq\frac{\epsilon}{(\frac{L\epsilon}{d})^d},\; M_J <  \frac{c_d}{\epsilon^2}\left(\frac{d}{\epsilon^2}\right)^d\log\left(\frac{1}{\epsilon}\right) \right\rbrace,\\ 
   \plongthree &\defn \left\lbrace B_J \; |\; B_J \in \plong,\; p_J < \frac{\epsilon}{(\frac{L\epsilon}{d})^d} \right\rbrace.
\end{align*}}
The boxes in both first and second types have large probability $p_J$ with respect to distribution $\D$ but differ in the number of unlabeled samples in $S$ lying in them. Here, $c_d$ is some constant depending on the dimension $d$. Finally, the boxes in third subtype $\plongthree$ have small probability $p_J$ with respect to distribution $\D$. Now, we can divide the total error of long boxes into error in these subtypes
\begin{small}
\begin{align}
\sum_{J:B_J \in \plong}p_J\diff_{\D_J}(\ft, f_{\D}^*) &= \underbrace{\sum_{J:B_J \in \plongone}p_J\diff_{\D_J}(\ft, f_{\D}^*)}_{E1} + \underbrace{\sum_{J:B_J \in \plongtwo}p_J\diff_{\D_J}(\ft, f_{\D}^*)}_{E2}+ \underbrace{\sum_{J:B_J \in \plongthree}p_J\diff_{\D_J}(\ft, f_{\D}^*)}_{E3} 
\label{eqn:combined-highd}
\end{align}
\end{small}
Now, we will argue about the contribution of each of the three terms above. 

\underline{Bounding $E3$.} Since there are at most $(\frac{L\epsilon}{d})^d$ long boxes and each of these boxes $B_J$ has probability $p_J$ upper bounded by $\frac{\epsilon}{(\frac{L\epsilon}{d})^d}$, the total probability combined in these boxes is at most $\epsilon$. Also, in the worst case, the loss can be 1. Hence, we get an upper bound of $\epsilon$ on $E_3$.

\underline{Bounding $E2$.} From Lemma~\ref{lemma:number_unlabeled_samples-highd}, we know that with success probability $\delta$, these boxes have total probability upper bounded by ${\epsilon}/{\delta}$. Again, the loss can be 1 in the worst case. Hence, we can get an upper bound of ${\epsilon}/{\delta}$ on $E_2$.

\underline{Bounding $E1$.}  
Let $F_J$ denote the event that $\diff_{\D_J}(\hat{f} _{S\cap B_J}, f_{\D_J}^*) > \epsilon$.  The expected error of boxes $B_J$ in $\plongone$ is then
{\small
\begin{align*}
\E[\sum_{J:B_J \in \plongone}p_J\diff_{\D}(\ft, f_{\D_J}^*)] &\stackrel{\1}{\leq} \E[\sum_{J:B_J \in \plongone}p_J\diff_{\D_J}(\hat{f}_{S\cap B_J}, f_{\D_J}^*)]\\ 
&= \sum_{J:B_J \in \plongone}p_J(\E[\diff_{\D_J}(\hat{f}_{S\cap B_J}, f_{\D_J}^*)|F_J]\Pr(F_J) + \E[\diff_{\D_J}(\hat{f}_{S\cap B_J}, f_{\D_J}^*)|\neg F_J]\Pr(\neg F_J))\\
&\stackrel{\2}{\leq} \sum_{J:B_J \in \plongone}p_J(1\cdot \epsilon + \epsilon \cdot 1) 
\leq 2\epsilon,
\end{align*}}
where step $\1$ follows by noting that $\ft = \hat{f}_{S\cap B_J}$ for all long boxes $B_J \in \plong$ and that $f^*_{\D_J}(x)$ is the minimizer of the error $\err_{\D_J}(f)$ over all $L$-Lipschitz functions, 
and step $\2$ follows since $\E[\diff_{\D_J}(\hat{f}_{S\cap B_J}, f_{\D_J}^*)|\neg F_J] \leq \epsilon$ by the definition of event $F_J$ and $\Pr(F_J) \leq \epsilon$ follows from a standard uniform convergence argument (detailed in Lemma~\ref{lem: good type 1 intervals-highd}).
Now, using Markov's inequality, we get that $E_1 \leq {2\epsilon}/{\delta}$ with failure probability at most $\delta$. 

Plugging the error bounds obtained in equations~\eqref{eqn:type2-error-highd} and ~\eqref{eqn:combined-highd} into equation~\eqref{eqn:total-error-highd} and setting $\delta = \frac{1}{20}$ establishes the required claim.

\paragraph{Label Query Complexity.} Observe that for any given query point $x^*$, $\ft(x^*)$ can be computed by only querying the labels of the box in which $x$ lies (in case if $x^*$ lies in a long box) or extension of the box in which $x^*$ lies (in case if $x^*$ lies in a short box) and hence, would only require $\frac{1}{\epsilon^2}(O(\frac{d}{\epsilon^2}))^d\log(\frac{1}{\epsilon})$ active label queries over the set $S$ of $\frac{1}{\epsilon^3}(O(\frac{L}{\epsilon}))^d\log(\frac{1}{\epsilon})$ unlabeled samples.
\end{proof}
Now, we will state the sample complexity bounds for estimating the error of the optimal function amongst the class of Lipschitz functions up to an additive error of $\epsilon$. 
\begin{restatable}{theorem}{theoremlipschitzerror-highd}
\label{thm:lipschitzerror-highd}
For any distribution $\mathcal{D}$ over $[0,1]^d\times[0,1]$, Lipschitz constant $L>0$ and parameter $\epsilon \in [0,1]$, Algorithm~\ref{alg3} uses $\frac{1}{\epsilon^4}(O(\frac{d}{\epsilon^2}))^d\log(\frac{1}{\epsilon})$ active label queries on $\frac{1}{\epsilon^3}(O(\frac{L}{\epsilon}))^d\log(\frac{1}{\epsilon})$ unlabeled samples from distribution $\D_x$ and produces an output $\widehat{\errd}(\Fl)$ satisfying
\begin{equation*}
|\widehat{\errd}(\Fl) - \errd( \Fl)| \leq \epsilon 
\end{equation*} 
with probability at least $\frac{1}{2}$.
\end{restatable}
%
%
The proof goes exactly like the proof for the corresponding theorem in the one-dimensional case (Theorem~\ref{thm:lipschitzerror}).

The following lemma proves that with enough unlabeled samples, a large fraction of long boxes have enough unlabeled samples in them which is eventually used to argue that they will be sufficient to learn a function which is approximately close to the optimal Lipschitz function over that box. 
\begin{lemma}
\label{lemma:number_unlabeled_samples-highd}
For any distribution $\D_x$, consider a set $S = \{x_1, x_2, \cdots, x_M\}$ of unlabeled samples where each sample $x_i\stackrel{\text{i.i.d.}}{\sim}\D_x$. Let $\mathcal{G}$ be the set of long boxes $\{B_J\}$ each of which satisfies $p_J = \Pr_{x \sim \D_x}(x \in B_J) \geq \frac{\epsilon}{(\frac{L\epsilon}{d})^d}$. Let $E_J$ denote the event that $\sum_{x_j \in S}\mathbb{I}[x_j \in B_J] < \frac{1}{\epsilon^2}(\frac{d}{\epsilon^2})^d\log(\frac{1}{\epsilon})$. Then, we have 
\begin{equation*}
    \sum_{B_J \in \mathcal{G}}p_J\mathbb{I}[E_J] \leq \frac{\epsilon}{\delta}
\end{equation*}
with failure probability atmost $\delta$ for $M = \frac{1}{\epsilon^3}(\Omega(\frac{L}{\epsilon}))^d\log(\frac{1}{\epsilon})$. 
\end{lemma}

\begin{proof}
For any box $B_J \in \mathcal{G}$, we have that $\E[\sum_{x_j \in S}\mathbb{I}[x_j \in B_J]] \geq \frac{c_d}{\epsilon^2}(\frac{d}{\epsilon^2})^d\log(\frac{1}{\epsilon})$ for some constant $c_d$ depending on the dimension $d$. Using Hoeffding inequality, we can get that $\Pr(E_J) \leq \epsilon$ for all boxes $B_J \in \mathcal{G}$. Calculating expectation of the desired quantity, we get
\begin{equation*}
    \E[\sum_{B_J \in \mathcal{G}}p_J\mathbb{I}[E_J]] = \sum_{B_J \in \mathcal{G}}p_J\Pr[E_J] \leq \epsilon\sum_{B_J \in \mathcal{G}}p_J \leq \epsilon
\end{equation*}

We get the desired result using Markov's inequality.
\end{proof}

%
%
%
%
%
The following lemma states that the probability of short boxes $\psh$ is small with high probability. Consider the case of uniform distribution. In this case, since the short boxes cover only $2\epsilon$ fraction of the domain $[0,1]^d$, their probability $\psh$ is upper bounded by $2\epsilon$. The case for arbitrary distributions holds because the boxes are chosen randomly.
\begin{lemma}
\label{lemma:prob_intervals-highd}
When we divide the $[0,1]^d$ domain into long and short boxes as in the preprocessing step (Algorithm~\ref{alg4}), then
\begin{equation*}
    \psh = \sum_{B_J \in \pshort}p_J \leq \frac{2\epsilon}{\delta}
\end{equation*}
with failure probability atmost $\delta$. 
\end{lemma}
%
%
%
\begin{proof}
Now, we consider the division of every dimension i.e. $[0,1]$ independently into alternating intervals of length $\frac{d}{L\epsilon}$ and $\frac{2}{L}$ with the offset chosen uniformly randomly from $\{0, 1, 2, \cdots, \frac{d}{2\epsilon}\}\frac{2}{L}$. For any set of fixed offsets chosen for the $d-1$ dimensions, the intervals of length $\frac{2}{L}$ chosen for the $dth$ dimension combined are disjoint in each of these divisions and together cover the entire $[0,1]^d$ and hence amount to probability mass $1$. Therefore, the total probability mass covered in the total $(\frac{d}{2\epsilon}+1)^d$ possible divisions is $d(\frac{d}{2\epsilon}+1)^{d-1}$.  Hence, there are at most $\delta$ fraction out of the total $(\frac{d}{2\epsilon}+1)^d$ cases where the short boxes have probability greater than $\frac{2\epsilon}{\delta}$. Hence, with probability $1-\delta$, the short boxes have probability upper bounded by $\frac{2\epsilon}{\delta}$. %
\end{proof}

\begin{lemma}
\label{lem: good type 1 intervals-highd}
Let $B_J \in \plongone$ be any long box of subtype 1. For the event \mbox{$F_J = \{\diff_{\D_J}(\hat{f}_{S \cap B_J}, f_{\D_J}^*) > \epsilon
\}$}, we have
\begin{equation*}
    \Pr(F_J) \leq \epsilon
\end{equation*}
\end{lemma}

\begin{proof}
We know that the covering number for the class of $d$-dimensional $L$-Lipschitz functions supported on the box $[0,l]^d$ is $(O(\frac{Ll}{\epsilon}))^d$ (Lemma~\ref{lemma:lipschitz_covering}). For, Lipschitz functions supported on long intervals of length $l=\frac{d}{L\epsilon}$ along each dimension, we get this complexity as $(O(\frac{d}{\epsilon^2}))^d$.  We know by standard results in uniform convergence, that the number of samples required for uniform convergence up to an error of $\epsilon$ and failure probability $\delta$ for all functions in a class $\mathcal{F}_L$ is $O((\frac{\text{Covering number of } \mathcal{F}_L)}{\epsilon^2})\log(\frac{1}{\delta}))$ (Lemma~\ref{lemma:lipschitz_sample_complexity}) and hence, for long boxes we get this estimate as $\frac{1}{\epsilon^2}(O(\frac{d}{\epsilon^2}))^d\log(\frac{1}{\delta})$. 
Given that there are $\frac{1}{\epsilon^2}(\Omega(\frac{d}{\epsilon^2}))^d\log(\frac{1}{\epsilon})$ samples in these boxes by design, the function learned using these samples is only additively worse than the function with minimum error using standard learning theory arguments with failure probability at most $\epsilon$. Thus, we get that 
$$\diff_{\D_J}(f_{S \cap B_J}, f_{\D_J}^*) \leq \epsilon \quad \forall B_J \in \plongone$$
with failure probability at most $\epsilon$. 
\end{proof}
\section{Nadaraya-Watson Estimator}
\label{appendix:kde}
In this section, we state the proof of Theorem~\ref{thm:kde-min-error} and the lemmas needed. 
But, first we state a theorem from \cite{backurs2018efficient} which will be crucial in the proof. The following theorem essentially states that for certain nice kernel $K(x,y)$, it is possible to estimate $\frac{1}{N}\sum_{i=1}^N K(x,x_i)$ with a multiplicative error of $\epsilon$ for any query $x$ efficiently. In this section, we will use $A\preceq B \text{ }(\text{or } A\succeq B)$ to denote that $ B-A\text{ }(\text{or } A-B)$ is a positive semidefinite matrix for two positive semidefinite matrices $A$ and $B$.
%
\begin{theorem} [Theorem 11 of \cite{backurs2018efficient}]
\label{thm:backurs2018efficient}
There exists a data structure that given a data set $P \subset \R^d$ of size $N$, using $O(dL2^{O(t)}\log(\frac{\Phi N}{\delta}))\frac{1}{\epsilon^2}\cdot N$ space and preprocessing time, for any $(L,t)$ nice kernel and a query $q \in \R^d$, estimates $KDF_P(q) = \frac{1}{|P|}\sum_{y\in P}k(q,y)$ with accuracy $(1\pm\epsilon)$ in time $O(dL2^{O(t)}\log(\frac{\Phi N}{\delta}))\frac{1}{\epsilon^2})$ with probability at least $1-\frac{1}{poly(N)}-\delta$.
\end{theorem}

The kernel used in this setting $K_A(x,y) = \frac{1}{1+||A(x-y)||_2^2} \ \forall A \in \mathcal{A}$ is $(4,2)$ smooth according to their definition of smoothness as shown in Definition 1 in \cite{backurs2018efficient} and hence can be computed efficiently. Moreover, as mentioned in \cite{backurs2018efficient}, it is also possible to remove the dependence on the aspect ratio and in turn achieve time complexity of $O(\frac{d}{\epsilon^2}\log(\frac{ N}{\delta}))$ with a preprocessing time of $O(\frac{dN}{\epsilon^2}\log(\frac{ N}{\delta}))$. Note that the data structure in \cite{backurs2018efficient} only depends on the smoothness properties of the kernel and hence, the same data structure can be used for simultaneously computing the kernel density for all kernels $K_A$ for all   
$A \in \mathcal{A}$. As a direct corollary of Theorem~\ref{thm:backurs2018efficient}, we obtain that it is possible to efficiently estimate $p_{S,A}(q,x_i)$ $\forall x_i\in S, q \in \R^d, A \in \mathcal{A}$ since multiplication by a constant still preserves the multiplicative approximation (Corollary~\ref{thm:backurs2018efficient-cor}). Now, we restate Theorem~\ref{thm:kde-min-error} and its proof. 
\theoremkdeminerror*
\begin{proof}
Let us assume that $A^*$ is the optimal matrix which minimizes the prediction error i.e. 
$$A^* = \argmin_{A \in \mathcal{A}} \E_{x\sim
\mathcal{D}}\sum_ip_{S,A}(x_i,x)|f(x_i)-f(x)|$$ 
Let us consider a set $\mathcal{A}_{\epsilon}$, an $\epsilon$-covering of the set of matrices $\mathcal{A}$ with size $T = |\mathcal{A}_{\epsilon}| = O(\frac{1}{\epsilon^d})$, that is, $$\mathcal{A}_{\epsilon}=\{A \in \R^{d\times d}\ |\ A_{i,j} = 0 \  \forall \ i\neq j \text{ and } A_{i,i}\in\{1,1+\epsilon,(1+\epsilon)^2,\cdots,2\} \  \forall i \in [d]\}. $$
From Lemma \ref{lemma:kde-approximate}, we know that it is sufficient to estimate $\min_{A \in \mathcal{A}_{\epsilon}} \E_{x\sim
\mathcal{D}}|\sum_ip_{S, A}(x_i,x)|f(x_i)-f(x)|$ up to an error of $\epsilon$ because the optimal error in $\mathcal{A}$ and $\mathcal{A}_{\epsilon}$ are within $\epsilon$ of each other. To estimate this, we  use the estimator from Algorithm~\ref{alg-kde}. 

Now, we will prove that the estimator is within $E_{x \sim \mathcal{D}}\sum_{i=1}^Np_{S,A}(x,x_i)|f(x_i)-f(x)| \pm \epsilon$ with high probability for all $A \in \mathcal{A}_{\epsilon}$. Let $E$ be the event that the estimators $\hat{p}_{S, A}$ from Algorithm \ref{alg-kde} all approximate $p_{S, A}$ up to a multiplicative error of $\epsilon$, that is,
\begin{equation*}
E = \{\hat{p}_{S, A}(z_i,\zt_i) \in p_{S, A}(z_i,\zt_i)[1-\epsilon, 1+\epsilon] \quad \forall i \in [M], \forall A \in \mathcal{A}_{\epsilon}\}.
\end{equation*}
Now, we will break the probability of the estimator $\hat{L}_{S, K_A}$ not being within $\epsilon$ close to the true value $L_{S,K_A}$ for all $A \in \mathcal{A}$. 
\begin{small}
\begin{equation*}
    \Pr(|\hat{L}_{S, K_A}  - L_{S, K_A} | > \epsilon) = \underbrace{\Pr(|\hat{L}_{S, K_A}  - L_{S, K_A} | > \epsilon| E)}_{T_1}\Pr(E) + \Pr(|\hat{L}_{S, K_A}  - L_{S, K_A} | > \epsilon | \neg E)\underbrace{\Pr(\neg E)}_{T_2}
\end{equation*}
\end{small}
\underline{Bounding $T_1$. } Computing the expectation of the estimator conditioned on the event $E$, we get that 
\begin{align*}
E[\hat{L}_{S,K_A}|E] &= \frac{1}{M}\sum_iE_{z_i\sim \mathcal{D}, \zt_i\sim p_{S,I}(z_i,\zt_i)}[\frac{\hat{p}_{S,A}(z_i, \zt_i)}{p_{S,I}(z_i, \zt_i)}|f(z_i)-f(\zt_i)|]\\
&= \frac{1}{M}\sum_iE_{z_i\sim \mathcal{D}}[\sum_{j=1}^N\hat{p}_{S,A}(z_i,x_j)|f(z_i)-f(x_j)|]\\
&= E_{z\sim \mathcal{D}}\sum_{j=1}^N\hat{p}_{S,A}(z,x_j)|f(z)-f(x_j)|\\
&\in E_{z\sim \mathcal{D}}\sum_{j=1}^N{p}_{S,A}(z,x_j)|f(z)-f(x_j)|[1-\epsilon, 1+\epsilon]\\
&\in L_{S,K_A}[1-\epsilon, 1+\epsilon]
\end{align*}
Now, we will show that the estimator is close to its expectation with high probability conditioned on the event $E$. Since we know that $|f(z_i)-f(\zt_i)| \leq 1$ and
\begin{equation*}
p_{S,A}(z_i,\zt_i) \leq 16p_{S,I}(z_i,\zt_i) \quad \forall z_i,\zt_i \in \R^d, \forall A \in \mathcal{A}, 
\end{equation*}
from Lemma \ref{lemma:kde-ineq}, we get that each entry of our estimator is bounded between $[0,16(1+\epsilon)]$. Hence, using Hoeffding's inequality we get that 
\begin{equation*}
\Pr(|\hat{L}_{S,K_A}-E[\hat{L}_{S,K_A}|E]| \geq \epsilon|E) \leq 2e^{\frac{-2M\epsilon^2}{{256(1+\epsilon)}^2}}.    
\end{equation*}
Hence, for $M = O(\frac{1}{\epsilon^2}\log(\frac{T}{\delta}))$ and union bound over the $\epsilon$ cover of size $T$, we get $T_1 \leq \delta$

\underline{Bounding $T_2$. } Using Corollary~\ref{thm:backurs2018efficient-cor} with $\delta$  as $\frac{\delta\epsilon^d}{M}$ and a union bound over the $M\cdot\frac{1}{\epsilon^d}$ computations of $\hat{p}_{S, A}$, we get that 
\begin{equation*}
    \Pr(\neg E) \leq \delta + \frac{1}{poly(N)}\frac{M}{\epsilon^d}
\end{equation*}

%
%
%
Combining the upper bounds for $T_1$ and $T_2$, 
\begin{equation*}
    \Pr(|\hat{L}_{S, K_A}  - L_{S, K_A} | > \epsilon) \leq 2\delta +  \frac{1}{poly(N)}\frac{M}{\epsilon^d}, 
\end{equation*}
for $M = O(\frac{1}{\epsilon^2}\log(\frac{T}{\delta}))$. 
Substituting $T = \frac{1}{\epsilon^d}$, we get a sample complexity of $M = O(\frac{1}{\epsilon^2}(d\log(\frac{1}{\epsilon})+\log(\frac{1}{\delta})))$ and failure probability of $2\delta+ \frac{1}{poly(N)\epsilon^{d+2}}(d\log(\frac{1}{\epsilon})+\log(\frac{1}{\delta}))$. 

\paragraph{Running Time.} For the running time, since we have
a total of $O(\frac{1}{\epsilon^2}(d\log(\frac{1}{\epsilon})+\log(\frac{1}{\delta}))$ samples and we have to take a sum over these samples for each $A \in \mathcal{A}_{\epsilon}$, we get a total running time of $O(\frac{1}{\epsilon^d}(\frac{1}{\epsilon^2}(d\log(\frac{1}{\epsilon})+\log(\frac{1}{\delta})))$ for this part. $p_{S,I}$ needs to be computed only once for each sample and this leads to a running time of $O(N(\frac{1}{\epsilon^2}(d\log(\frac{1}{\epsilon})+\log(\frac{1}{\delta})))$. We also compute an estimator $\hat{p}_{S,A}$ for each of the sample pair for each $A \in \mathcal{A}_{\epsilon}$ and thus from Corollary~\ref{thm:backurs2018efficient-cor}, we get a preprocessing time of $O(\frac{Nd}{\epsilon^2}\log(\frac{NM}{\delta\epsilon^d}))$ and $O(\frac{Md}{\epsilon^{d+2}}\log(\frac{NM}{\delta\epsilon^d}))$ time in computing the estimator for each $A \in \mathcal{A}_{\epsilon}$. 
This completes the proof of the theorem.
%
%
%
%
%
%
%
%
%
\end{proof}
The following lemma used in the proof of Theorem~\ref{thm:kde-min-error} states that it is sufficient to consider an $\epsilon$-net of the set of matrices $\mathcal{A}$ to approximate the minimum error up to an additive error of $\epsilon$. 
\begin{lemma}
\label{lemma:kde-approximate}
Let us consider a set $\mathcal{A}_{\epsilon}$, an $\epsilon$-covering of the set of matrices $\mathcal{A}$ i.e. $$\mathcal{A}_{\epsilon}=\{A \in \R^{d\times d}\ |\ A_{i,j} = 0 \  \forall \ i\neq j \text{ and } A_{i,i}\in\{1,1+\epsilon,(1+\epsilon)^2,\cdots,2\} \  \forall i \in [d]\}$$ 
Then the minimum prediction error for $A \in \mathcal{A}_{\epsilon}$ additively approximates the minimum prediction error for $A \in \mathcal{A}$ i.e.
$$\big|\min_{A \in \mathcal{A}_{\epsilon}} \E_{x\sim
\mathcal{D}}\sum_i p_{S,A}(x_i,x)|f(x_i)-f(x)|-\min_{A \in \mathcal{A}} \E_{x\sim
\mathcal{D}}\sum_ip_{S,A}(x_i,x)|f(x_i)-f(x)|\big| \leq 15\epsilon$$
\end{lemma}

\begin{proof}
We first show that that if $(1+\epsilon)^{-1}A_2 \preceq A_1 \preceq A_2(1+\epsilon)$, then $$\big|\E_{x\sim \mathcal{D}}\sum_i p_{S,A_1}(x_i,x)|f(x_i)-f(x)|-\E_{x\sim \mathcal{D}}|\sum_i p_{S,A_2}(x_i,x)|f(x_i)-f(x)|\big| \leq 15\epsilon$$
Note that the loss difference can also be written as $ \E_{x\sim \mathcal{D}}\sum_i |p_{S,A_1}(x_i,x)-p_{S,A_2}(x_i,x)||f(x_i)-f(x)|$. 
Now from Lemma~\ref{lemma:kde-ineq}, we know that if $(1+\epsilon)^{-1}A_2 \preceq A_1 \preceq A_2(1+\epsilon)$, then $(1+\epsilon)^{-4}p_{S,A_2}(x_i,x) \leq p_{S,A_1}(x_i,x) \leq p_{S,A_2}(x_i,x)(1+\epsilon)^4$. Using this,  we get that 
\begin{multline*}
 E_{x\sim \mathcal{D}}\sum_i|p_{S,A_1}(x_i,x)-p_{S,A_2}(x_i,x)||f(x_i)-f(x)|\\
  \begin{aligned}
    &\leq  E_{x\sim \mathcal{D}}\sum_i|p_{S,A_2}(x_i,x)(1+\epsilon)^4-p_{S,A_2}(x_i,x)||f(x_i)-f(x)|\\
&\leq 15\epsilon E_{x\sim \mathcal{D}}\sum_i |p_{S,A_2}(x_i,x)||f(x_i)-f(x)|\\
&\leq 15\epsilon E_{x\sim \mathcal{D}}\sum_i |p_{S,A_2}(x_i,x)| \leq 15\epsilon
  \end{aligned}
\end{multline*}
Hence, the loss difference is also bounded by $15\epsilon$.
Now, we will prove that $\exists A \in \mathcal{A}_\epsilon$ such that $(1+\epsilon)^{-1}A \preceq A^* \preceq A(1+\epsilon)$ where $A^* = \argmin\E_{x\sim \mathcal{D}}\sum_i |p_{S,A}(x_i,x)||f(x_i)-f(x)|$. This will be sufficient to show that the minimum error for the set $\mathcal{A}_\epsilon$ and $\mathcal{A}$ differ from each other by an additive error at most $\epsilon$. Consider each of the diagonal entries of $A$ to be the value in $\{1,1+\epsilon,(1+\epsilon)^2,\cdots,2\}$ closest to the corresponding entry of $A^*$. It is easy to see that $(1+\epsilon)^{-1}A \preceq A^* \preceq A(1+\epsilon)$. 
\end{proof}
The following lemma states that if two matrices $A_1$ and $A_2$ are multiplicatively close to each other in terms of all their eigenvalues, then the corresponding probabilities for any query point $x$ and any data point $x_i$ are also multiplicatively close to each other.
\begin{lemma}
\label{lemma:kde-ineq}
For any two matrices $A_1,A_2 \in \R^{d\times d}$, if $\frac{1}{1+\epsilon}A_2 \preceq A_1 \preceq A_2(1+\epsilon)$, then $$\frac{1}{(1+\epsilon)^4}p_{S,A_2}(x_i,x) \leq p_{S,A_1}(x_i,x) \leq p_{S,A_2}(x_i,x)(1+\epsilon)^4 \quad \forall x \in \R^d, x_i\in S.$$
\end{lemma}
\begin{proof}
Using $(1+\epsilon)^{-1}A_2 \preceq A_1 \preceq A_2(1+\epsilon)$, we get that
\begin{alignat}{3}
\centermathcell{||A_2(x_i-x)||\frac{1}{1+\epsilon}} &\leq&  \centermathcell{||A_1(x_i-x)||} &\leq& \centermathcell{||A_2(x_i-x)||(1+\epsilon)}\nonumber\\
\centermathcell{1+||A_2(x_i-x)||^2\frac{1}{(1+\epsilon)^2}} &\leq& \centermathcell{1+||A_1(x_i-x)||^2} &\leq& \centermathcell{1+||A_2(x_i-x)||^2(1+\epsilon)^2}\nonumber\\
\centermathcell{\frac{1}{(1+\epsilon)^2}(1+||A_2(x_i-x)||^2)} &\leq& \centermathcell{\text{   }1+||A_1(x_i-x)||^2\text{   }} &\leq& \centermathcell{\text{   }(1+||A_2(x_i-x)||^2)(1+\epsilon)^2\text{   }}\nonumber\\
\centermathcell{\frac{1}{(1+\epsilon)^2}K_{A_2}(x_i,x)} &\leq& \centermathcell{K_{A_1}(x_i,x)} &\leq& \centermathcell{K_{A_2}(x_i,x)(1+\epsilon)^2}\label{eqn:kde-ineq}
\end{alignat}
Hence, using the inequality in equation~\ref{eqn:kde-ineq}, we get that
\begin{alignat*}{3}
\centermathcell{\text{   }\frac{1}{(1+\epsilon)^4}\frac{K_{A_2}(x_i,x)}{\sum_{x_i \in S} K_{A_2}(x_i,x)}\text{   }} &\leq& \centermathcell{\text{   }\frac{K_{A_1}(x_i,x)}{\sum_{x_i\in S} K_{A_1}(x_i,x)}\text{   }} &\leq& \centermathcell{\text{   }\frac{K_{A_2}(x_i,x)}{\sum_{x_i \in S} K_{A_2}(x_i,x)}(1+\epsilon)^4\text{   }}\\
\centermathcell{\frac{1}{(1+\epsilon)^4}p_{S,A_2}(x_i,x)} &\leq& \centermathcell{p_{S,A_1}(x_i,x)} &\leq& \centermathcell{p_{S,A_2}(x_i,x)(1+\epsilon)^4}
\end{alignat*}
This completes the proof of the lemma.
\end{proof}
\end{document}